\documentclass[]{interact}

\pdfoutput=1 

\usepackage{epstopdf}
\usepackage[caption=false]{subfig}

\usepackage[natbibapa,nodoi]{apacite} 
\theoremstyle{plain}
\newtheorem{theorem}{Theorem}[section]
\newtheorem{lemma}[theorem]{Lemma}
\newtheorem{corollary}[theorem]{Corollary}

\theoremstyle{definition}

\theoremstyle{remark}

\usepackage{hyperref}
\usepackage{tabularx}
\usepackage{xcolor}
\usepackage[latin1]{inputenc}
\usepackage{etoolbox} 
\usepackage{amsmath}
\DeclareMathOperator{\sgn}{sgn}

\begin{document}


\newtoggle{anonymous}
\togglefalse{anonymous}

\title{Optimal coding and the origins of Zipfian laws}

\iftoggle{anonymous}{}
{ 
  \author{
  \name{Ramon Ferrer-i-Cancho\textsuperscript{a}\thanks{CONTACT Ramon Ferrer-i-Cancho. Email: rferrericancho@cs.upc.edu}, Christian Bentz\textsuperscript{b,c} and Caio Seguin\textsuperscript{d}}
  \affil{
         \textsuperscript{a}Complexity \& Quantitative Linguistics Lab, LARCA Research Group, Departament de Ci\`encies de la Computaci\'o, Universitat Polit\`ecnica de Catalunya, Campus Nord, Edifici Omega, Jordi Girona Salgado 1-3. 08034 Barcelona, Catalonia (Spain). ORCiD: 0000-0002-7820-923X. rferrericancho@cs.upc.edu. \\
         \textsuperscript{b}URPP Language and Space, University of Z\"{u}rich, Freiestrasse 16, CH-8032 Z\"{u}rich, Switzerland. ORCiD: 0000-0001-6570-9326. chris@christianbentz.de \\
         \textsuperscript{c}DFG Center for Advanced Studies ``Words, Bones, Genes, Tools'', University of T\"{u}bingen, R\"{u}melinstra{\ss}e 23, D-72070 T\"{u}bingen, Germany. ORCiD: 0000-0001-6570-9326. \\
         \textsuperscript{d}Melbourne Neuropsychiatry Centre,  The University of Melbourne and Melbourne Health, Melbourne, VIC 3010, Australia. ORCiD: 0000-0001-9384-6336. caio.seguin@unimelb.edu.au.}
  }
}

\maketitle

\begin{abstract}
The problem of compression in standard information theory consists of assigning codes as short as possible to numbers.
Here we consider the problem of optimal coding -- under an arbitrary coding scheme -- and show that it predicts Zipf's law of abbreviation, namely a tendency in natural languages for more frequent words to be shorter. We apply this result to investigate optimal coding also under so-called non-singular coding, a scheme where unique segmentation is not warranted but codes stand for a distinct number.
Optimal non-singular coding predicts that the length of a word should grow approximately as the logarithm of its frequency rank, which is again consistent with Zipf's law of abbreviation. Optimal non-singular coding in combination with the maximum entropy principle also predicts Zipf's rank-frequency distribution. 
Furthermore, our findings on optimal non-singular coding challenge common beliefs about random typing. It turns out that random typing is in fact an optimal coding process, in stark contrast with the common assumption that it is detached from cost cutting considerations. 
Finally, we discuss the implications of optimal coding for the construction of a compact theory of Zipfian laws more generally as well as other linguistic laws.
\end{abstract}





\section{Introduction}

Zipf's law of abbreviation states that more frequent words tend to be shorter \citep{Zipf1949a}. Its widespread presence in human languages \citep{Bentz2016a}, and the growing evidence in other species \citep{Ficken1978a, Hailman1985, Ferrer2009g,Ferrer2012a,Ferrer2012d,Luo2013a,Heesen2019a,Demartsev2019a,Favaro2020a,Huang2020a}, calls for a theoretical explanation. 
The law of abbreviation has been interpreted as a manifestation of compression \citep{Ferrer2012d}, assigning strings as short as possible to represent information, a fundamental problem in information theory, and coding theory in particular \citep{Cover2006a}. Here we aim to investigate compression as a fundamental principle for the construction of a compact theory of linguistic patterns in natural communication systems \citep{Ferrer2015b}. We explore the relationship between compression and Zipf's law of abbreviation, as well as other regularities such as Zipf's law for word frequencies. The latter states that $p_i$, the probability of $i$-th most frequent word, follows \citep{Zipf1949a},   
\begin{equation}
p_i \approx i^{-\alpha},
\label{Zipfs_law_equation}
\end{equation}
where $\alpha$ is the exponent (a parameter of the distribution) that is assumed to be about $1$ \citep{Ferrer2004a}. \citet{Zipf1949a} referred to Equation \ref{Zipfs_law_equation} as the {\em rank-frequency distribution}.

In standard information theory, codes are strings of symbols 
from a certain alphabet of size $N$ which are used to represent discrete values from a set of $V$ elements, e.g., natural numbers \citep{Borda2011a}. Suppose that the codes have minimum length $l_{min}$ (with $l_{min}=1$ by default). 
For example, if the alphabet is formed by letters $a$ and $b$, the possible codes are 
\begin{equation}
a, b, aa, ab, ba, bb, aaa, aab, aba, abb, baa,...
\label{shortest_words_equation}
\end{equation} 
As a set of discrete values one may have natural numbers,
\begin{equation*}
1,2,3,4,5,6,7,8,9,10,11,...
\end{equation*} 
For simplicity, we assume that we wish to code for natural numbers from $1$ to $V$. 
These numbers should be interpreted as what one wishes to code for or as indices or identifiers of what one actually wishes to code for. Therefore, if one wished to code for $V$ different objects that are not numbers from 1 to $V$, one should label each object with a distinct number from 1 to $V$.

In that framework, the problem of compression consists of assigning codes to natural numbers from $1$ to $V$ in a way to minimize the mean length of the codes, defined as \citep{Cover2006a}
\begin{equation}
L  = \sum_{i=1}^V p_i l_i,
\label{mean_code_length_equation}
\end{equation}
where  
$p_i$ is the probability of the $i$-th number and $l_i$ is the length of its code in symbols. 
The standard problem of compression consists of minimizing $L$ with the $p_i$'s as a given, and under some coding scheme \citep{Cover2006a}. Roughly speaking, a coding scheme is a constraint on how to translate a number into a code in order to warrant successful decoding, namely retrieving the original number from the code from the receiver's perspective. In the examples of coding that will follow, we assume that one wishes to 
code numbers from $1$ to $6$ on strings from an alphabet of two letters $a$ and $b$. Table \ref{unconstrained_coding_example_table} shows an example of unconstrained coding (no scheme is used). 
The coding in that example is optimal because all strings have minimum length but it is not very useful because each string has three numbers as possible interpretations.

\begin{table}
\caption{\label{unconstrained_coding_example_table} An example of optimal unconstrained coding of numbers from $1$ to $6$ on strings from an alphabet of two letters $a$ and $b$.}
\centering
\begin{tabular}{rr}
  Number & Code \\ \hline
  $1$ & $a$ \\
  $2$ & $a$ \\
  $3$ & $a$ \\
  $4$ & $b$ \\
  $5$ & $b$ \\
  $6$ & $b$
\end{tabular}
\end{table}

Table \ref{non-singular_coding_example_table} shows an example of so-called \textit{non-singular} coding, meaning that a unique code is assigned to each number. Thus, every code has only one possible interpretation. If we assigned the string $aa$ to more than one number, the coding would not be non-singular. The example in Table \ref{unconstrained_coding_example_table} is not non-singular either.  
\begin{table}
\caption{\label{non-singular_coding_example_table} An example of non-singular coding of numbers from $1$ to $6$ on strings from an alphabet of two letters $a$ and $b$.}
\centering
\begin{tabular}{rr}
  Number & Code \\ \hline
  $1$ & $aa$ \\
  $2$ & $ab$ \\
  $3$ & $a$ \\
  $4$ & $b$ \\
  $5$ & $ba$ \\
  $6$ & $bb$
\end{tabular}
\end{table}
In the standard problem of compression, the alphabet is also a given. Therefore, $L$ is minimized with $N$ constant.

The problem of compression can be related to human languages in two ways: either we think of the numbers as representing word types (distinct words), or as representing meaning types (distinct meanings).
In the former case, codes stand for distinct word types, in the latter case, they stand for distinct meanings.  
If numbers represent word types, then a typical application is to solve the problem of optimal \textit{recoding}, namely reducing the length of words as much as possible without losing their distinctiveness.  
If we consider numbers to represent meaning types, then human languages  do not perfectly fit the non-singular coding scheme due to polysemy (the same word types can have more than one meaning). 
However, non-singularity is convenient for language {\em a priori} because it reduces the cost of communication from the listeners perspective \citep{Zipf1949a,Piantadosi2012a} as well as the cost of vocabulary learning in children \citep{Casas2019a}.
Optimization pressures in both ways -- shortening of codes, on the one hand, and reducing polysemy (eventually leading to non-singular coding), on the other -- are likely to coexist in real languages, as suggested by experiments \citep{Kanwal2017a}. See \citet[Section 5.2]{Ferrer2015b} for a possible formalization based on a generalization of standard coding theory.

The information theory concepts introduced above have a direct correspondence with popular terms used in research on language optimization. The non-singular scheme implies least effort for the listener, in G. K. Zipf's terms \citep{Zipf1949a}.
Zipf's law of abbreviation was explained as the result of combining two pressures \citep{Kanwal2017a}:
{\em accuracy}, i.e. avoiding ambiguity, and {\em efficiency}, i.e. using word forms as short as possible. Communicating with maximum accuracy (no ambiguity) is equivalent to the non-singular scheme. Compression (the minimization of $L$) is equivalent to efficiency. 
 
A further coding scheme, which is central to information theory, is \textit{uniquely decodable} coding, namely, non-singular coding with unique segmentation. That is, when codes are concatenated without a separator, e.g., space, there should be only one way of breaking the sequence into codes. Uniquely decodable codes are hence a subset of non-singular codes (Figure \ref{hiearchy_of_coding_schemes_figure}).
 
\begin{figure}
\begin{center}
\includegraphics[scale = 0.5]{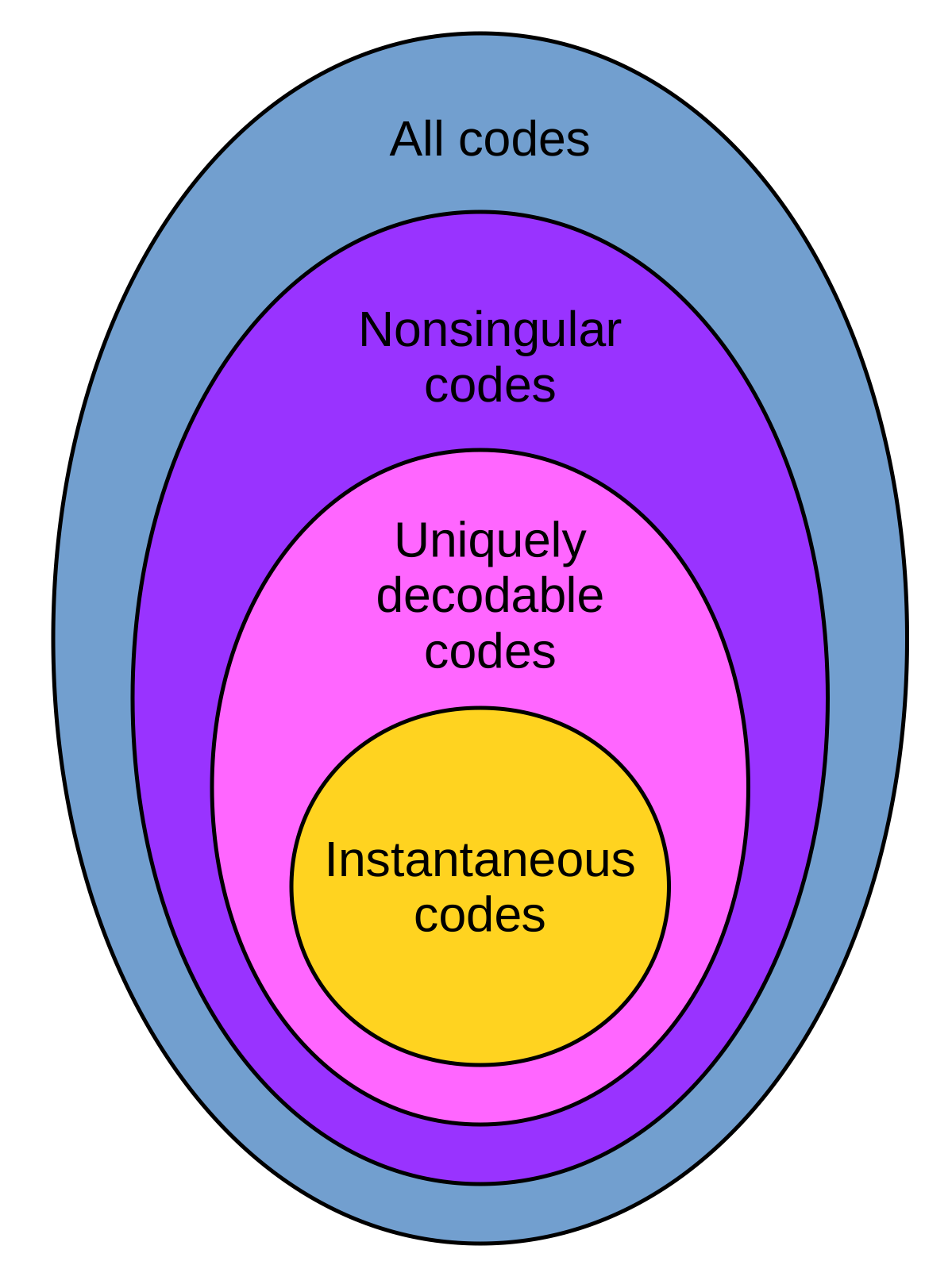}
\end{center}
\caption{\label{hiearchy_of_coding_schemes_figure} Classes of codes. Adapted from \citet[p. 106]{Cover2006a}. Instantaneous codes, that are not described in the main text, are codes such that there is no string in the coding table that matches the beginning of another string (totally or partially). An example of instantaneous code would be the binary representation of numbers from 1 to 6 in the examples of the article. }
\end{figure}

The coding in Table \ref{non-singular_coding_example_table} is not uniquely decodable because the string $baba$ can be interpreted as $4343$, $55$, etc. In contrast, Table \ref{uniquely_decodable_encoding_example_table} shows a coding that is uniquely decodable. The string $baba$ can here only be interpreted as $12$.
  
\begin{table}
  \caption{\label{uniquely_decodable_encoding_example_table} An example of uniquely decodable coding of numbers from $1$ to $6$ on strings from an alphabet of two letters $a$ and $b$ using Elias gamma encoding (a coding procedure where the code itself tells its length, turning segmentation straightforward \citep[p. 199]{Elias1975a}. }
\centering 
\begin{tabular}{rr}
  Number & Code \\ \hline
  $1$ & $b$ \\
  $2$ & $aba$ \\
  $3$ & $abb$ \\
  $4$ & $aabaa$ \\
  $5$ & $aabab$ \\
  $6$ & $aabba$
\end{tabular}
\end{table}

It is easy to see that written English, when written without spaces, is often not uniquely decodable. {\em together} can be read as both a single word and also {\em to get her} \citep{McMillan1956a}. {\em Godisnowhere} illustrates the same problem: it can be read either as {\em God is nowhere} or as {\em God is now here}. Similar examples can be found in spoken English or other languages. However, unique decodability would be generally convenient for segmenting speech easily \citep{Romberg2010a}. Again, unique decodability is a listener's requirement, who has to be able to retrieve the codes and the corresponding numbers when the codes are produced in a row (lacking spaces or silences in between them).

\begin{table}
\caption{\label{optimal_non-singular_coding_example_table} Optimal non-singular coding of numbers from 1 to 6 on strings consisting of symbols $a$ and $b$. Notice that codes are assigned to frequency ranks. }
\centering
\begin{tabular}{rr}
  Rank & Code \\ \hline
  $1$ & $a$ \\
  $2$ & $b$ \\
  $3$ & $aa$ \\
  $4$ & $ab$ \\
  $5$ & $ba$ \\
  $6$ & $bb$
\end{tabular}
\end{table}

Now, suppose that we assign a frequency rank to each number (the most frequent number has rank 1, the 2nd most frequent number has rank 2, and so on). In his pioneering research, Mandelbrot considered the problem of compression implicitly, by assuming that word types are the numbers to code, and wrote that \citep[p. 365]{Mandelbrot1966} {given any prescribed multiset of word probabilities, the average number of letters per word} ($L$ in our notation above) {\em is minimized if the list of words ranked by decreasing probability, coincides with the list of the $V$ shortest letter sequences, ranked by increasing number of letters} 
(as in Table \ref{optimal_non-singular_coding_example_table} for the case of only two letters). In the language of information theory, he addressed the problem of compression under the scheme of optimal non-singular coding. To our knowledge, a formal proof of the optimality of his coding procedure is still lacking. 
In fact, information theoretic research has generally neglected the problem of optimal non-singular coding since then, and instead focused on uniquely decodable encoding. The reasons for this are three-fold: 
\begin{itemize}
\item
The primary target of standard information theory are artificial devices (not human brains or natural communication systems).
\item
The hard segmentation problem arising when non-singular codes are concatenated without separators (word delimiters).
\item
The waste of time/space when separators are added to facilitate segmentation over these codes \citep[p. 105]{Cover2006a}.  
\end{itemize}
These considerations may have prevented information theory from providing simple explanations to linguistic laws. 

The remainder of the article is organized as follows. Section \ref{optimal_coding_section} presents a generalization of the problem of compression that predicts the law of abbreviation under an arbitrary coding scheme. This type of compression problem is used to prove that non-singular coding consists of assigning a string as short as possible (preserving non-singularity) to each number following frequency ranks in ascending order -- as expected by \citet{Mandelbrot1966}. As an example, the coding in Table \ref{optimal_non-singular_coding_example_table} satisfies this design, while that of Table \ref{uniquely_decodable_encoding_example_table} does not (in the latter, all codes are unnecessarily long from non-singular coding perspective except for rank 1). 
In case of optimal non-singular coding, Section \ref{optimal_coding_section} shows that $l_i$ is an increasing logarithmic function of $i$, the frequency rank when $N > 1$, and a linear function of $i$ when $N=1$, giving an exact formula in both cases. This prediction is a particular case of Zipf's law of abbreviation. 

The logarithmic relation between length and frequency rank that results from optimal non-singular coding is crucial: it provides a justification for the logarithmic constraint that is needed by the most parsimonious derivation of Zipf's rank-frequency distribution based on the maximum entropy principle \citep{Visser2013a}. For this reason, Section \ref{maximum_entropy_principle_section} revisits Mandelbrot's derivation of Zipf's distribution combining optimal non-singular coding, and the maximum entropy (maxent) principle \citep{Mandelbrot1966}. This adds missing perspectives to his original analysis, and illustrates the predictive capacity of optimal non-singular coding with regards to linguistic laws. Although the distribution of word frequencies is power-law-like, an exponential distribution is found for other linguistic units, e.g. part-of-speech tags \citep[p. 116-122]{Tuzzi2010a},
colors \citep{Ramscar2019a}, kinship terms \citep{Ramscar2019a} and verbal alternation classes \citep{Ramscar2019a}. 
Beyond texts, exponential distributions are found in first names in the census or social security records \citep{Ramscar2019a}. Non-singular coding and maxent can shed light on the emergence of these two types of distributions.  
In particular, Section \ref{maximum_entropy_principle_section} shows how the combination of the maximum entropy principle and optimal non-singular coding predicts two different distributions of ranks depending on the value of $N$. When $N > 1$, it predicts Equation \ref{Zipfs_law_equation}. When $N = 1$, it predicts a geometric distribution of ranks, namely,
\begin{equation}
p_i = q (1-q)^{i-1},
\label{geometric_distribution_equation}
\end{equation}
where $q$ is a parameter between 0 and 1. In addition, such a geometric distribution may arise from suboptimal coding when $N > 1$.

Section \ref{random_typing_section} then challenges the long-standing believe that random typing constitutes evidence that Zipfian laws (Zipf's rank-frequency law and Zipf's law of abbreviation) can be derived without any optimization or cost-cutting consideration \citep{Miller1957,Li1998,Kanwal2017a,Chaabouni2019a}: random typing emerges as an optimal non-singular coding system in disguise. In addition, we investigate various properties of random typing, applying results on optimal coding from Section \ref{optimal_coding_section}, and providing a simple analytical expression for the relationship between the probability of a word and its rank -- a result that \citet{Mandelbrot1966} believed to be impossible to obtain.

Section \ref{discussion_section} discusses the implications for empirical research on linguistic laws and how compression, optimal coding and maximum entropy can contribute to the construction of a general but compact theory of linguistic laws.

\section{Optimal coding}

\label{optimal_coding_section}

Here we investigate a generalization of the problem of compression, where $L$ (Equation \ref{mean_code_length_equation}) is generalized as mean energetic cost, i.e.
\begin{equation}
\Lambda = \sum_{i=1}^V p_i \lambda_i,
\label{energy_supplementary_equation}
\end{equation}
and $p_i$ and $\lambda_i$ are, respectively, the probability and the energetic cost of the $i$-th type. 
Without any loss of generality, suppose that the types to be coded are sorted nonincreasingly, i.e. 
\begin{equation}
p_1 \geq p_2 \geq ... \geq p_V, 
\label{ordering_of_probabilities_equation} 
\end{equation}
Roughly speaking, a nonincreasing order is the outcome of sorting in decreasing order. We refer to it as nonincreasing instead of decreasing because, strictly, a decreasing order can only be obtained if all the values are distinct.

The generalization is two-fold. First, $\lambda_i = g(l_i)$, where $g$ is a strictly monotonically increasing function of $l_i$. 
Second, $l_i$ is generalized as a magnitude, namely, a positive real number.
When $g(l_i) = l_i$ and $l_i$ is the length in symbols of the alphabet, $\Lambda$ becomes $L$ (Equation \ref{mean_code_length_equation}), the mean code length of standard information theory \citep{Cover2006a}. The generalization function $g$ follows from other research on the optimization of communication where the energetic cost of the distance between syntactically related words in a linear arrangement is assumed to be a strictly monotonically increasing function of that distance \citep{Ferrer2013e}. The goal of $g$ is abstracting away from the translation of some magnitude (word length or distance between words) into a real energetic cost. Here we investigate the minimization of $\Lambda$ when the $p_i$'s are constant (given) as in the standard problem of compression, 
where the magnitudes are lengths of strings following a certain scheme \citep{Cover2006a}.

\subsection{Unconstrained optimal coding}

\label{unconstrained_optimal_coding_section}

The solution to the minimization of $\Lambda$ when no further constraint is imposed is that all types have minimum magnitude, i.e.
\begin{equation}
l_i = l_{min} \mbox{~for~} i = 1, 2,..., V.
\label{trivial_minimization_equation}
\end{equation}
Then $\Lambda$ is minimized absolutely when $l_{min}=0$, the smallest possible magnitude.

Now suppose that $l_i$ is a length as in standard information theory. The condition in Equation \ref{trivial_minimization_equation} implies that all types are assigned the empty string. Then the coding fails to be non-singular (for $V > 1$).
If empty strings are not allowed then $l_{min} = 1$. In that case, optimal coding will produce codes that are not non-singular if $N < V$ (as in Table \ref{unconstrained_coding_example_table}). One may get codes that are non-singular by increasing $N$. However, recall that $N$ is constant in the standard problem of compression.


First, we will investigate the problem of compression (minimization of $\Lambda$) when the lengths are generalized to magnitudes (positive real numbers) that belong to a given multiset. Second, we will apply the results to the problem of compression in the non-singular scheme (the multiset contains the lengths of all distinct strings).  

\subsection{Optimal coding with given magnitudes}

\label{optimal_coding_with_given_magnitudes_section}

Suppose that we wish to minimize $\Lambda$ where the $l_i$'s are taken from a multiset ${\cal L}$ of real positive values with $|{\cal L}| \geq V$. For instance, the values could be the length in symbols of the alphabet or the duration of the type.
An assignment of elements of ${\cal L}$ to the $l_i$'s consists of sorting the elements of ${\cal L}$ forming a sequence and 
assigning to each $l_i$ the $i$-th element of the sequence. For an assignment, only the $V$ first elements of the sequence matter. After an assignment, the $l_i$'s define a subset of ${\cal L}$, i.e. 
\begin{equation*}
\{l_1, ..., l_i,...,l_V\} \subseteq {\cal L}.
\end{equation*}
Therefore, ${\cal L}$ is a given in addition to the $p_i$'s. ${\cal L}$ allows one to capture arbitrary constraints on word length, beyond the traditional coding schemes (e.g., non-singular coding or uniquely decodable encoding). Perceptibility and distinguishability factors may prevent the use of very short strings, even under a uniquely decodable scheme. Phonotactics (a branch of phonology) shows that not all possible combinations of phonemes are present in a language. Certain phonemes or combinations are harder (if not impossible) to articulate or perceive. See \citet[chapters 3 and 4]{Akmajian1995a} for an overview of these concepts and constraints from linguistics. 

This problem of compression is more general than the compression problem in standard information theory because:
\begin{itemize}
\item
$l_i$ is generalized as a magnitude, namely a positive real number. The strings, even when the magnitude is a length, are irrelevant. 
\item
In case the magnitudes are string lengths, the non-singular coding scheme is obtained defining ${\cal L}$ as the lengths of all the different strings that can be formed. Similarly, in case of uniquely decodable coding, the string lengths have to allow one to find strings that produce them while preserving the constraints of the scheme.   
\end{itemize}
These two generalization allow us to shed light on the origins of Zipf's law of abbreviation in human languages, where words do not match perfectly the constraints of traditional schemes, as well as in other species, where the coding scheme is unknown and the magnitude is measured as a time duration, namely a positive real value (e.g., \citet{Semple2010a,Heesen2019a}). Moreover, it is conceivable that certain natural communication systems do not build signs by combining elementary units (such as phonemes or syllables as in human languages) -- as assumed by standard information theory -- but rather hollistically. Such cases could be implemented as strings of length 1 and their magnitude could be a real number indicating their expected duration. 

When $|{\cal L}| = V$, there are as many different assignments as different sequences that can be produced from ${\cal L}$.
When $|{\cal L}| \geq V$, the solution to the problem of compression consists of finding $\Lambda_{min}$, the minimum value of $\Lambda$, and the assignments that achieve the minimum, over all the 
\begin{equation*}
\frac{|{\cal L}|!}{(|{\cal L}| - V)!}
\end{equation*}
assignments of elements of ${\cal L}$ to the $l_i$'s. We will show that $\Lambda$ is minimized exclusively by all the assignments from orderings of the elements of ${\cal L}$ such that the $V$ first elements are the $V$ smallest elements of ${\cal L}$ sorted in {\em nondecreasing} order (we refer to it as nondecreasing instead of increasing because, strictly, an increasing order can only be obtained if all the values are distinct). There is only one assignment if the values in ${\cal L}$ are distinct and 
$|{\cal L}| = V$.

Suppose that $n_c$ is the number of concordant pairs of an assignment. 
$(p_i, l_i)$ and $(p_j, l_j)$ 
are said to be concordant if 
\begin{equation}
\sgn(p_i - p_j) \sgn(l_i - l_j) = 1,
\label{concordant_equation}
\end{equation}
where $\sgn$ is the sign function, i.e. 
\begin{equation*}
\sgn(x) = \left\{ 
         \begin{array}{c}
            \frac{x}{|x|} \mbox{~if~} x \neq 0 \\
            0 \mbox{~if~} x = 0.    
         \end{array} 
         \right.
\end{equation*}
Equation \ref{concordant_equation} is equivalent to 
\begin{equation*}
\sgn(p_i - p_j) = \sgn(l_i - l_j) \neq 0.
\end{equation*}

The following lemma gives a crucial necessary condition of optimal configurations:

\begin{lemma}
$\Lambda = \Lambda_{min}$ implies 
that the sequence $l_1$,...,$l_i$,...,$l_V$ is sorted in nondecreasing order, i.e. 
$n_c = 0$ over 
\begin{equation*}
(p_1, l_1),...,(p_i, l_i),...(p_V, l_V), 
\end{equation*}
because the sequence $p_1$,...,$p_i$,...,$p_V$ is sorted in nonincreasing order.
\label{absence_of_concordant_pairs_lemma}
\end{lemma}
\begin{proof}
We will proof the contrapositive, namely that $n_c > 0$ implies $\Lambda > \Lambda_{min}$ adapting arguments in previous work \citep{Ferrer2012d}. 
Let the pair $(p_i, l_i)$ and $(p_j, l_j)$ be concordant (then $i \neq j$) and such that $1 \leq i,j \leq V$. Without any loss of generality, suppose that $i < j$. Then $p_i > p_j$ by Equation \ref{ordering_of_probabilities_equation} (the case $p_i = p_j$ is excluded as the pair is concordant) and $l_i > l_j$ because the pair is concordant. If we swap $l_i$ and $l_j$, 
  then $\Lambda$ will become 
  \begin{eqnarray*}
  \Lambda' & = & \Lambda - p_i \lambda_i - p_j \lambda_j + p_i \lambda_j + p_j \lambda_i \\
         & = & \Lambda + (p_i - p_j) (\lambda_j - \lambda_i)
  \end{eqnarray*}
  and then the difference between the final and the initial value of $\Lambda$ becomes  
  \begin{eqnarray*}
  \Delta & = & \Lambda' - \Lambda \\ 
         & = & (p_i - p_j) (\lambda_j - \lambda_i).
  \end{eqnarray*}
  It is easy to see that $\Lambda > \Lambda_{min}$ as we wished because $\Delta <0$. Recall that, in this context, $p_i > p_j$ and $l_i > l_j$ (as explained above) and that $g$ is a strictly monotonically increasing function (notice that also $\Delta <0$ when $1 \leq j < i \leq V$).
\end{proof}

An assignment stemming from sorting the $V$ smallest elements of ${\cal L}$ in nondecreasing order (increasing order if the $V$ smallest elements of ${\cal L}$ are distinct) is equivalent to one where $n_c = 0$. The following theorem expresses it formally:

\begin{theorem}
\label{optimal_coding_theorem}
$\Lambda = \Lambda_{min}$ if and only if two conditions are met
\begin{itemize}
\item[1.]
$l_1$,...,$l_i$,...,$l_V$ are the $V$ smallest elements of ${\cal L}$. 
\item[2.]
The sequence $l_1$,...,$l_i$,...,$l_V$ is sorted in nondecreasing order, i.e. 
$n_c = 0$ over 
\begin{equation*}
(p_1, l_1),...,(p_i, l_i),...(p_V, l_V), 
\end{equation*}
because the sequence $p_1$,...,$p_i$,...,$p_V$ is sorted in nonincreasing order.
\end{itemize}
\end{theorem}
 
\begin{proof}
We proceed proving each direction of the equivalence separately.
\begin{enumerate}
\item
$\Lambda = \Lambda_{min}$ implies conditions 1 and 2 \\
  We will prove the contrapositive, 
namely that the failure of condition 1 or 2 implies $\Lambda > \Lambda_{min}$. 
  \begin{enumerate}
  \item
  Suppose that condition 1 fails. Then there is an element $l'$ in ${\cal L}\setminus \{l_1,...,l_i,...,l_V\}$ such that $l' < \max(l_1,...,l_i,...,l_V)$, where $\setminus$ is the multiset difference operator. Suppose that $k$ is the index of a magnitude such that $1 \leq k \leq V$ and $l_k > l'$. Assigning $l'$ to $l_i$, $\Lambda$ will decrease strictly because $l_k > l'$. Thus, the original value of $\Lambda$ satisfied $\Lambda > \Lambda_{min}$.  
  \item
  Suppose that condition 2 fails. Then $\Lambda > \Lambda_{min}$ by the contrapositive of Lemma \ref{absence_of_concordant_pairs_lemma}.
  \end{enumerate}

\item
Conditions 1 and 2 imply $\Lambda = \Lambda_{min}$ \\
We will show the contrapositive, 
namely that $\Lambda > \Lambda_{min}$ implies that condition 1 or 2 fails.
  $\Lambda > \Lambda_{min}$ can happen when condition 1 fails, as we have seen above. Suppose that condition 1 does not fail. Can we conclude that condition 2 fails? 
  Let $l_i^{min}$ and $\lambda_i^{min}$ be the values of $l_i$ and $\lambda_i$, respectively, in some minimum assignment, namely one yielding $\Lambda = \Lambda_{min}$.
  By Lemma \ref{absence_of_concordant_pairs_lemma}, the sequence $l_1^{min}$,...,$l_i^{min}$,...,$l_V^{min}$ is sorted in nondecreasing order and its corresponding number of concordant pairs is $n_c^{min} = 0$.
  Notice that $\Lambda > \Lambda_{min}$ implies that the $V$ smallest values of ${\cal L}$ are not identical (otherwise $\Lambda = \Lambda_{min}$ for any assignment satisfying condition 1). With this clarification in mind, it is easy to see that there must be some $i$ such that $\lambda_i > \lambda_i^{min}$, or equivalently, $l_i > l_i^{min}$. If that did not happen, then one would have $\lambda_j \leq \lambda_j^{min}$ for each $j$ such that $1 \leq j \leq V$ and then $\Lambda \leq \Lambda_{min}$, contradicting $\Lambda > \Lambda_{min}$. Crucially, such particular $i$ prevents the $l_i$'s from having the non-decreasing order that is defined by the $l_i^{min}$'s, leading to $n_c > 0$ by condition 1 and $n_c^{min} = 0$, as we wished. 
\end{enumerate}
\end{proof}

The Kendall $\tau$ correlation between the $p_i$'s and the $l_i$'s is \citep{Conover1999a}
\begin{equation*}
\tau(p_i, l_i) = \frac{n_c - n_d}{{V \choose 2}},
\end{equation*}
where $n_d$ is the number of discordant pairs. $(p_i, l_i)$ and $(p_j, l_j)$ 
are said to be discordant if 
\begin{equation*}
\sgn(p_i - p_j) \sgn(l_i - l_j) = -1.
\end{equation*}
or, equivalently,
\begin{equation*}
\sgn(p_i - p_j) = -\sgn(l_i - l_j) \neq 0.
\end{equation*}
In our context,
\begin{equation*}
\tau(p_i, l_i) = \frac{1}{{V \choose 2}} \sum_{i < j} \sgn(p_i - p_j) \sgn(l_i - l_j).
\end{equation*}

An implication of optimal coding (minimum $\Lambda$) is that $\tau(p_i, l_i)$ cannot be positive. Formally:
\begin{corollary}
$\Lambda = \Lambda_{min}$ implies $\tau(p_i, l_i) \leq 0$ with equality if and only if $n_d = 0$.
\label{Kendal_correlation_corollary}
\end{corollary}
\begin{proof}
By Lemma \ref{absence_of_concordant_pairs_lemma} $\Lambda = \Lambda_{min}$ implies $n_c = 0$ and then 
\begin{equation*}
\tau(p_i, l_i) = -\frac{n_d}{{V \choose 2}}.
\end{equation*}
Since $n_d \geq 0$ one has $\tau(p_i, l_i) \leq 0$, with equality if and only if $n_d = 0$.
\end{proof}

\subsection{Optimal non-singular coding}

Under the scheme 
of uniquely decodable codes, standard information theory tells us that the minimization of $L$ leads to \citep{Cover2006a}
\begin{equation}
l_i \propto \lceil -\log_N p_i \rceil, 
\label{optimal_coding_supplementary_equation}
\end{equation} 
which is indeed a particular case of Zipf's law of abbreviation. This corresponds to the minimization of $\Lambda$ with $g$ as the identity function in our framework.  
Here we wish to minimize $\Lambda$ with $l_i$ as the length of the $i$-th most frequent type when only the $p_i$'s are prescribed under the non-singular coding scheme (Figure \ref{hiearchy_of_coding_schemes_figure}). 

Under non-singular coding, the set of available strings consists of all the different strings of symbols that can be built with an alphabet of size
$N$. There are $N^l$ different strings of length $l$. Let $S$ be the infinite sequence of these strings sorted by increasing length (the relative ordering of strings of the same length is arbitrary). If empty strings are not allowed, the strings in positions $1$ to $N$ have length $1$, the strings in positions $N+1$ to $N+N^2$ have length 2, and so on as in \ref{shortest_words_equation} for $N = 2$. 

\begin{corollary}
Optimal non-singular coding consists of assigning the $i$-th string of $S$ to the $i$-th most probable type for $1 \leq i \leq V$.
\label{optimal_non-singular_coding_corollary}
\end{corollary}
\begin{proof}
We define ${\cal L}$ as the multiset of the lengths of the strings in $S$. 
As there is a one-to-one correspondence between an element of ${\cal L}$ and an available string, the application of theorem \ref{optimal_coding_theorem} with $g$ as the identity function gives that the optimal coding is such that 
\begin{itemize}
\item
The sequence $l_1,...,l_i,...,l_V$ contains the $V$ smallest lengths, and then comprises the codes that are the shortest possible strings. 
\item
$l_1,...,l_i,...,l_V$ is sorted in nondecreasing order, and then the $i$-th type is assigned the $i$-th shortest string.
\end{itemize}
\end{proof}


\subsection{Length as a function of frequency rank in optimal non-singular coding}

We aim to derive the relationship between the rank of a type (defined according to its probability) and its length in case of optimal non-singular codes for $N \geq 1$. 
Suppose that $p_i$ is the probability of the $i$-th most probable type and that $l_i$ is its length. The following lemma addresses a generalization of the problem:

\begin{lemma}
If rank $i$ is assigned the shortest possible string that has length $l_{min}$ or greater then
\begin{equation}
l_i = \left\{
         \begin{array}{ll}    
         \left\lceil \log_N \left((1-1/N)i + N^{l_{min}-1}\right) \right\rceil & \mbox{~for~} N > 1 \\ 
         i + l_{min} - 1 & \mbox{~for~} N = 1. \\
         \end{array}
      \right.
\label{optimal_non-singular_coding_equation}
\end{equation} 
\end{lemma}
\begin{proof}
Then the largest rank of types of length $l$ is  
\begin{equation*}
i = \sum_{k=l_{min}}^l N^k.
\end{equation*}
When $N > 1$, we get
\begin{equation*}
i = \frac{N^{l+1}-N^{l_{min}}}{N-1}
\end{equation*}
and equivalently 
\begin{equation*}
N^l = \frac{1}{N}[(N-1)i + N^{l_{min}}].
\end{equation*}
Taking logs on both sides of the equality, one obtains
\begin{equation*}
l = \frac{\log \left( \frac{1}{N}[(N-1)i + N^{l_{min}}] \right) }{\log N}.
\end{equation*}
The result can be generalized to any rank of types of length $l$ as
\begin{equation} 
l = \left\lceil \frac{\log \left( \frac{1}{N}[(N-1)i + N^{l_{min}}] \right) }{\log N} \right\rceil. 
\label{itermediate_equation}
\end{equation}
Changing the base of the logarithm to $N$, one obtains 
\begin{equation*} 
l = \left\lceil \log_N \left((1-1/N)i + N^{l_{min}-1}\right) \right\rceil. 
\end{equation*}
Alternatively, Equation \ref{itermediate_equation} also yields 
\begin{eqnarray*}
l & = & \left\lceil \frac{\log [(N-1)i + N^{l_{min}}]}{\log N} - 1 \right\rceil\\
  & = & \left\lceil \log_N [(N-1)i + N^{l_{min}}] \right\rceil - 1.
\end{eqnarray*}
The case $N = 1$ is trivial, one has $l = i + l_{min} - 1$. 
Therefore, the length of the $i$-th most probable type follows Equation \ref{optimal_non-singular_coding_equation}.	
\end{proof}

The previous arguments allow one to conclude:

\begin{corollary}
In case of optimal coding with non-singular codes, the length of the $i$-th most probable type follows Equation \ref{optimal_non-singular_coding_equation} with $l_{min} = 1$.
\end{corollary}
When $N > 1$, one obtains
\begin{equation*}
l_i = \left\lceil \log_N \left((1-1/N)i + 1 \right) \right\rceil,
\end{equation*}
the same conclusion was reached by \citet{Sudan2006a} though lacking a detailed explanation.

\subsection{Relationships with other mathematical problems}

We have investigated a problem of optimal coding where magnitudes stem from a given multiset of values. The problem is related to other mathematical problems outside coding theory. Notice that $\Lambda$ can be seen as a scalar product of two vectors, i.e. $\vec{p} = \{p_1,...,p_i,...,p_V\}$ and $\vec{\lambda} = \{\lambda_1,...,\lambda_i,...,\lambda_V\}$ and $L$ as a scalar product of $\vec{p}$ and 
$\vec{l} = \{l_1,...,l_i,...,l_V\}$. 
When $|{\cal L}| = V$ the problem is equivalent to minimizing the scalar (or dot) product of two vectors (of positive real values) over all the permutations of the content of each vector \citep{Aadam2016a}.
By the same token, 
the problem is equivalent to minimizing the Pearson correlation between $\vec{p}$ and $\vec{\lambda}$ when the content (but not the order) of each vector is preserved. Recall that the Pearson correlation between $\vec{p}$ and $\vec{\lambda}$ can be defined as \citep{Conover1999a}
\begin{equation}
r(\vec{p}, \vec{\lambda}) = \frac{\vec{p} \cdot \vec{\lambda} - \mu_p \mu_\lambda}{\sigma_p \sigma_\lambda},
\label{Pearson_correlation_equation}
\end{equation}
where $\mu_x$ and $\sigma_x$ are, respectively, the mean and the standard deviation of vector $\vec{x}$. 

The link with Pearson correlation goes back to the original coding problem: such a correlation has been used to find a concordance with the law of abbreviation that is in turn interpreted as a sign of efficient coding \citep{Semple2010a}. According to Equation \ref{Pearson_correlation_equation}, such a correlation turns out to be a linear transformation of the cost function. Put differently, minimizing $\Lambda$ with prescribed $p_i$'s and with 
$\lambda_i$ as the identity function (as it is customary in standard coding theory), is equivalent to minimizing the Pearson correlation at constant mean and standard deviation of both probabilities and magnitudes. Therefore, the Pearson correlation is a measure of the degree of optimization of a system when these means and standard deviations are constant (it is implicit that the standard deviations are not zero, otherwise the Pearson correlation is not defined).

\section{The maximum entropy principle}

\label{maximum_entropy_principle_section}

Now we turn onto the question of making a safe prediction on the distribution of word ranks in case of optimal non-singular coding. The maximum entropy principle states that \citep{Kesavan2009a} \\
~\\
{\em Out of all probability distributions consistent with a given set of constraints, the distribution
with maximum uncertainty should be chosen.} \\ 
~\\
The distribution of word frequencies has been derived via maximum entropy many times with similar if not identical methods \citep{Mandelbrot1966,Naranan1992b,Naranan1993,Ferrer2003c,Liu2008c,Baek2011a,Visser2013a}. Depending on the study, the target  
was Zipf's rank-frequency distribution, Equation \ref{Zipfs_law_equation}, \citep{Mandelbrot1966,Naranan1993, Liu2008c} or its sister law with frequency as the random variable \citep{Naranan1992b,Ferrer2003c}, stating that the $n_f$, the number of words of frequency $f$, satisfies approximately 
\begin{equation*}
n_f \approx f^{-\beta}
\end{equation*}
with $\beta \approx 2$ \citep{Zipf1949a,Moreno2016a}. In some cases, maximum entropy is used as an explanation for the ubiquity of power-law-like distributions, with Zipf's law for word frequencies or its sister as a particular case \citep{Baek2011a, Visser2013a}.
For simplicity, here we revisit the essence of the principle focusing on how our results on optimal non-singular coding can be used to derive different rank distributions.

The maximum entropy principle allows one to obtain a distribution that maximizes the entropy of probability ranks, namely, 
\begin{equation*}
H = - \sum_{i=1}^{V} p_i \log p_i
\end{equation*}
under certain constraints on cost over the $i$'s and a couple of elementary constraints on the $p_i$'s, i.e. $p_i \geq 0$ and 
\begin{equation*}
\sum_{i=1}^{V} p_i = 1.
\end{equation*} 
See \citet{Kapur1992a} and \cite{Harremoes2001a} for an overview.
For simplicity, we assume a single non-elementary cost constraint, namely $L$, as defined in Equation \ref{mean_code_length_equation}.
For simplicity, we assume that $V$ is not finite. See \citet{Visser2013a} for an analysis of the case of more than one non-elementary constraint and a comparison of
the finite versus infinite case. See \citet{Harremoes2001a} for some critical aspects of the traditional application of maximum entropy.

In our simple setup, the method leads to distributions of the form
\begin{equation}
p_i = \frac{e^{-\alpha l_i}}{Z},
\label{maximum_entropy_probability_equation}
\end{equation}
where $\alpha$ is a Lagrange multiplier and 
\begin{equation*}
Z = \sum_{j=1}^{\infty} e^{-\alpha l_j }
\end{equation*}
is the partition function.
In case of optimal non-singular coding, we have two cases. If $N > 1$ then $l_i \approx \log_N i$ for sufficiently large $N$ (Equation \ref{optimal_non-singular_coding_equation}),
which transforms Equation \ref{maximum_entropy_probability_equation} into a zeta distribution, i.e.  
\begin{equation}
p_i = \frac{1}{Z} i^{-\alpha}
\label{zeta_distribution_equation}
\end{equation}
while the partition function becomes
\begin{equation*}
Z = \sum_{j=1}^{\infty} j^{-\alpha},
\end{equation*}
namely the Riemann zeta function. The zeta distribution is an approximation to Zipf's law for word frequencies.

When $N = 1$ then $l_i =  i$ (Equation \ref{optimal_non-singular_coding_equation} with $l_{min} = 1$),
which transforms Equation \ref{maximum_entropy_probability_equation} into an exponential distribution of word frequencies, i.e. 
\begin{equation}
p_i = \frac{1}{Z} e^{-\alpha i}
\label{exponential_distribution_equation}
\end{equation}
while 
\begin{equation*}
Z = \sum_{j=1}^{\infty} e^{-\alpha i}.
\end{equation*}
Applying the same arguments, it is possible to obtain an exponential distribution via maximum entropy for $N > 1$ if $l_i = i$. In that case, however, the coding would be non-singular (every type would be coded with a string of distinct length) but would not be optimal.
Equation \ref{exponential_distribution_equation} matches the exponential-like distribution that is found for certain linguistic units \citep{Tuzzi2010a,Ramscar2019a}. 
In sum, this distribution may result, according to the maximum entropy principle, from either optimal or suboptimal coding.

Although Equation \ref{exponential_distribution_equation} is for a discrete random variable, it has the form of the popular exponential distribution for continuous random variables. That equation actually matches the definition of the customary geometric distribution in Equation \ref{geometric_distribution_equation}.
To see it, notice that $Z$ is the summation of a geometric series where the first term $a$ and the common factor $r$ are the same, i.e. $a=r=e^{-\alpha}$. Therefore, assuming $|r|<1$, i.e. $\alpha > 0$, 
\begin{eqnarray*}
Z & = & \frac{a}{1-r} \\
  & = & \frac{e^{-\alpha}}{1 - e^{-\alpha}}.
\end{eqnarray*} 
Then equation \ref{exponential_distribution_equation} can be rewritten equivalently as
\begin{equation}
p_i = \frac{1 - e^{-\alpha}}{e^{-\alpha}} (e^{-\alpha})^i.
\label{new_exponential_distribution_equation}
\end{equation}
The substitution $q = 1 - e^{-\alpha}$ transforms Equation \ref{new_exponential_distribution_equation} into the customary definition of a geometric distribution in Equation \ref{geometric_distribution_equation} as we wished. 
  

\section{The optimality of random typing}

\label{random_typing_section}

The results on optimal coding above allow one to unveil the optimality of typing at random, assuming that the space bar is hit with a certain probability and that letters are equally likely \citep{Miller1957}. It has been argued many times that random typing reproduces Zipf's rank-frequency distribution (e.g. \citet{Miller1957,Miller1963,Li1992b,Suzuki2004a}). In particular, Miller concluded that the law 
{\em ``can be derived from simple assumptions that do not strain one's credulity (unless the random placement of spaces seems incredible), without appeal to least effort, least cost, maximal information, or any other branch of the calculus of variations. The rule is a simple consequence of those intermittent silences which we imagine to exist between successive words.'' \citep{Miller1957}.}
Similarly, \citet{Li1998} argued that {\em ``random typing shows that a random process can mimic a cost-cutting process, but not purposely.''} A similar view is found in reviews of Zipf's law for word frequencies, where optimization and random typing are considered to be different mechanisms \citep{Mitzenmacher2003a,Newman2004a}. The view of random typing as detached from cost reduction is also found in research on the origins of Zipf's law of abbreviation \citep{Kanwal2017a,Chaabouni2019a}.  
Leaving aside the problem of the poor fit of random typing to their original target, i.e. the distribution of word frequencies \citep{Ferrer2009a,Ferrer2009b}, these views are also problematic because random typing and least cost are not really independent issues. We will show it through the eye of the problem of compression. 

The optimality of random typing can be seen in two ways. One through recoding, namely replacing each word it produces by another string so as to minimize $L$ under the non-singular coding scheme. The other - indeed equivalent - consists of supposing that random typing is used to code for numbers whose probability matches that of the words produced by random typing. In both cases, we will show that the value of $L$ of a random typing process cannot be reduced and thus it is optimal. Put differently, we will show that there is no non-singular coding system that can do it more efficiently (with a smaller $L$) than random typing.

It is easy to see that the strings that random typing produces are optimal according to Corollary \ref{optimal_non-singular_coding_corollary}. 
Recall that the probability of a ``word'' $w$ of length $l$ in random typing is \citep[p. 838]{Ferrer2009a}
\begin{equation}
p_l(w) = \left( \frac{1-p_s}{N} \right) ^{l} \frac{p_s}{(1-p_s)^{l_{min}}},
\label{probability_of_a_word_in_random_typing_equation}
\end{equation}
where $l$ is the length of $w$, $p_s$ is the probability of producing the word delimiter (a whitespace), $N$ is the size of the alphabet that the words consist of ($N > 0$) and $l_{min}$ is the minimum word length ($l_{min} \geq 0$). 
Hereafter we assume for simplicity that $0 < p_s < 1$. If $p_s = 0$, strings never end. If $p_s = 1$, all the strings have length $l_{min}$ and then random typing has to be analyzed following the arguments for unconstrained optimal coding in Section \ref{unconstrained_optimal_coding_section}.

We will show that after sorting nondecreasingly all possible strings of length at least $l_{min}$ that can be formed with $N$ letters, the $i$-th most likely type of random typing receives the $i$-shortest string.
First, Equation \ref{probability_of_a_word_in_random_typing_equation} indicates that all words of the same length are equally likely and $p_{l+1}(w) \geq p_l(w)$ for $l \geq l_{min}$ because $p_s$, $l_{min}$ and $N$ are constants. Therefore, the ranks of words of length $l$ are always larger than those of words of length $l+1$. Keeping this property in mind, words of the same length are assigned an arbitrary rank. 
Second, $p_l(w) > 0$ for all the $N^l$ different words of length $l$ that can be formed. Therefore, all available strings of a given length are used.
The optimality of random typing for $N = 2$ and $l_{min} = 1$ can be checked easily in Table \ref{binary_random_typing_table}. The exact relationship between rank and length in random typing will be derived below.

\begin{table}
\caption{\label{binary_random_typing_table} The probability ($p_i$), the length ($l_i$) of the $i$-th most frequent string (code) or random typing with $N = 2$ and $l_{min} = 1$. $l_i$ is calculated via Equation \ref{optimal_non-singular_coding_equation} with $N=2$ and $l_{min}=1$. $p_i$ is calculated applying $l_{min}=1$, $N=2$ and $l_i$ to Equation \ref{probability_versus_rank_in_random_typing_equation}. }
\centering
\begin{tabular}{rrrr}
Code & $i$ & $l_i$ & $p_i$ \\ \hline
a    & 1 & 1 & $p_s/2$ \\
b    & 2 & 1 & $p_s/2$ \\
aa   & 3 & 2 & $(1-p_s)p_s/4$ \\
ab   & 4 & 2 & $(1-p_s)p_s/4$ \\
ba   & 5 & 2 & $(1-p_s)p_s/4$ \\
bb   & 6 & 2 & $(1-p_s)p_s/4$ \\
aaa  & 7 & 2 & $(1-p_s)^2 p_s/8$ \\
...  & ... & ... & ... 
\end{tabular}
\end{table}
 
Random typing satisfies a particular version of Zipf's law of abbreviation where the length of a word ($l$) is a linear function of its probability ($p$), i.e.
\begin{equation}
l = a \log p + b, 
\label{law_of_abbreviation_in_random_typing_equation}
\end{equation}
where $a$ and $b$ are constants ($a < 0$). Namely, the probability of a word is determined by its length (the characters constituting the words are irrelevant),
Equation \ref{probability_of_a_word_in_random_typing_equation} allows one to express $l$ as a function of $p(w)$. Rearranging the terms of Equation \ref{probability_of_a_word_in_random_typing_equation}, taking logarithms, and replacing $p(w)$ by $p$,  one recovers Equation \ref{law_of_abbreviation_in_random_typing_equation} with 
\begin{equation*}
a = \left( \log \frac{1-p_s}{N} \right)^{-1}
\end{equation*} 
and 
\begin{equation*}
b = a \log \frac{(1-p_s)^{l_{min}}}{p_s}.
\end{equation*}

Does random typing also satisfy Zipf's law for word frequencies (Equation \ref{Zipfs_law_equation})? Mandelbrot was aware ``{\em that the relation between rank and probability is given by a step function}'' for the random typing model we have considered here,  but he argued that {\em ``such a relation cannot be represented by any simple analytic expression''} \citep[p. 364]{Mandelbrot1966}. 
Knowing that random typing is optimal from the standpoint of non-singular coding it is actually possible to obtain a simple analytic expression for $p_i$, the probability  that random typing produces a word of rank $i$. Replacing $p_l(w)$ by $p_i$ and $l$ by $l_i$, Equation \ref{probability_of_a_word_in_random_typing_equation} becomes
\begin{equation}
p_i = \left( \frac{1-p_s}{N} \right) ^{l_i} \frac{p_s}{(1-p_s)^{l_{min}}},
\label{probability_versus_rank_in_random_typing_equation}
\end{equation}
where $l_i$ the length of the word of rank $i$ that is given by Equation \ref{optimal_non-singular_coding_equation}. 
To our knowledge, this is the first exact Equation for $p_i$. In previous research, only approximate expressions for $p_i$ have been given \citep{Miller1957,Miller1963,Mandelbrot1966,Li1992b}. 
These non-rigorous approximations correspond to the Zipf-Mandelbrot law,	
\begin{equation}
p_i \propto (i + b)^{-\alpha},
\label{Zipf-Mandelbrots_law_equation}
\end{equation}
a generalization of Zipf's law (Equation \ref{Zipfs_law_equation}) with an additional parameter $b > 0$ \citep{Mandelbrot1966}, that is actually a smoothed version of Equation \ref{probability_versus_rank_in_random_typing_equation}.
If ranks are unbounded as in the random typing model, the Zipf-Mandelbrot law can be defined exactly as 
\begin{equation} 
p_i = \frac{1}{\zeta(\alpha,b)}(i + b)^{-\alpha},
\label{Hurwitz_zeta_distribution_equation}
\end{equation}
where
\begin{equation}
\zeta(\alpha,b) = \sum_{i = 0}^{\infty} (i + b)^{-\alpha}
\end{equation} 
is the Hurwitz zeta function. Equation \ref{Hurwitz_zeta_distribution_equation} becomes the definition of the zeta distribution (Equation \ref{zeta_distribution_equation}) when $b = 1$. 

Figure \ref{random_typing_figure} a) shows $p_i$ versus $i$ for $N=26$ and $p_s=0.18$ applying Equation \ref{probability_versus_rank_in_random_typing_equation}. These are the parameters that \citet{Miller1957} used in his classic article on random typing to mimic English. 
The stepwise shape, that is missing in the Zipf-Mandelbrot law (equations \ref{Zipf-Mandelbrots_law_equation} and \ref{Hurwitz_zeta_distribution_equation}), can be smoothed by introducing a bias towards certain letters \citep{Li1992b,Ferrer2009b} as in the original setup all letters are equally likely. Reducing $N$ as much as possible will also smooth the shape (reducing $N$ is a particular case of bias that consists of turning 0 the probability of certain symbols). Figure \ref{random_typing_figure} b) shows the smoothing effect of $N=2$ (corresponding to the examples given in Table~\ref{binary_random_typing_table}). Notice that $N$ cannot be reduced further: we have shown above that $N=1$ transforms the distribution of ranks of random typing into a geometric distribution.

\begin{figure}
\center
\includegraphics[scale = 0.7]{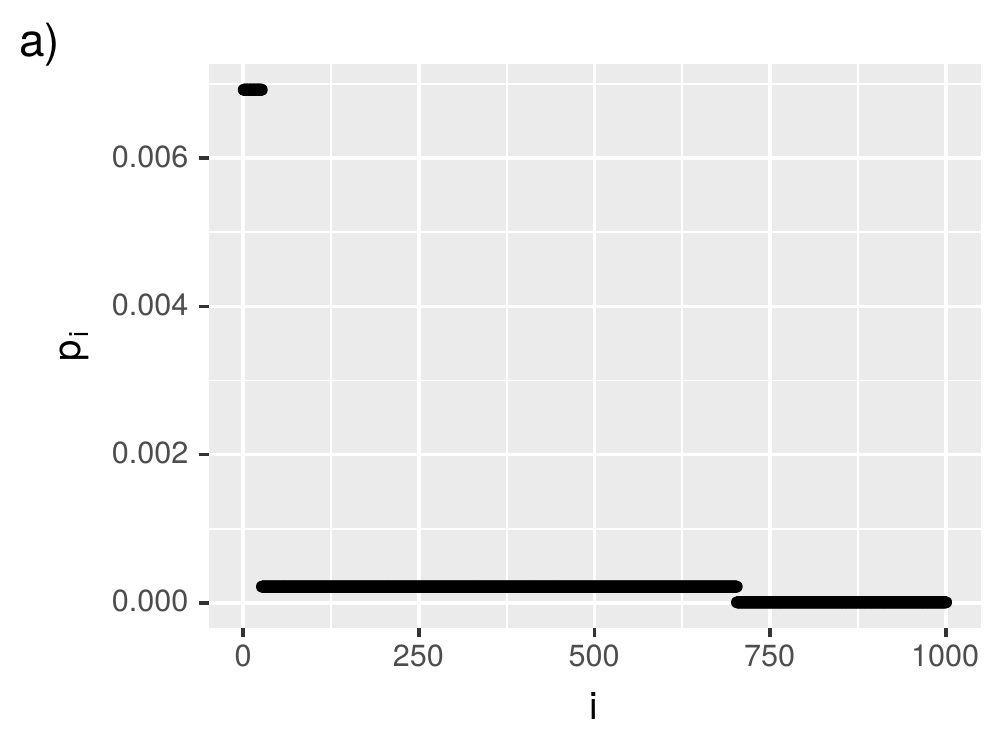}
\includegraphics[scale = 0.7]{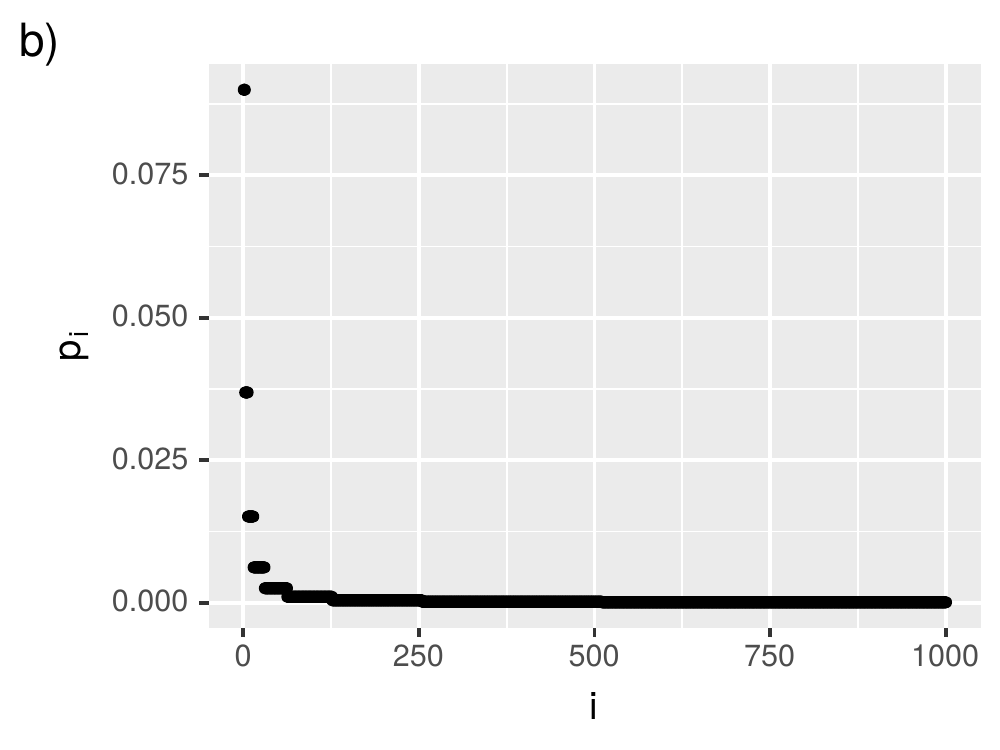}
\caption{\label{random_typing_figure} $p_i$ the probability of a word of rank $i$ produced by random typing with $p_s = 0.18$, $l_{min}=1$ up to rank $i_{max} = 1000$ (in  the standard random typing model, the maximum rank is infinite). a) $N=26$. b) $N=2$.}
\end{figure}
  
\section{Discussion}

\label{discussion_section}

In his pioneering research, Zipf found a tendency of more frequent words to be shorter. He termed this observation the law of abbreviation \citep{Zipf1949a}. However, he never proposed a functional dependency or mathematical model for the relationship between frequency and length. 

Here we have filled a gap in standard information theory concerning optimal non-singular coding, that predicts $l_i \approx \log i$, where $i$ is the probability rank. This result complements the well-known relationship $l_i \approx - \log p_i$ predicted by optimal uniquely decodable coding \citep{Cover2006a}.    
Derivations of a logarithmic relationship between the length of a word and its probability rank can be found in classic work \citep{Mandelbrot1966,Rapoport1982}. However, our derivation is novel in the sense of providing a general exact formula (not an approximation; covering $N \geq 1$ and $l_{min} \geq 0$) and involving optimal non-singular coding in the argument. It was clear to Mandelbrot that 
{\em ``given any prescribed set of word probabilities, the average number of letters per words is minimized if the list of words, ranked by decreasing probability, coincides with the list of the $V$ shortest letter sequences''} \citep[p. 365]{Mandelbrot1966} but he never provided an exact formula for the relationship between $l_i$ and $i$ as far as we know. Indeed, he actually thought it was impossible \citep{Mandelbrot1966}. Likewise, Rapoport 
did not take information theoretic optimality considerations into account
and simply stated that {\em ``we shall want the shortest words to be the most frequent''} \citep[p. 9]{Rapoport1982}.

Traditionally, quantitative linguistics research has been based on the fit of power-law-like models \citep{Sigurd2004a,Strauss2007a}. Surprisingly, the predictions of information theory reviewed above have largely been neglected. The problem concerns not only the relationship between length and frequency but also parallel quantitative linguistics research where frequency is replaced by the frequency rank (see \citet[p. 274]{Strauss2007a} and references therein). 
Some notable exceptions are discussed in the following. 

In the work by \citet{Hammerl1990a}, both the relationship $l_i \approx \log p_i$ and $l_i \approx \log i$ are considered. 
He explains that Guiraud (in 1959) derived $l_i \approx \log i$ by {\em ``purely combinatorial considerations, where all possible combinations of letters in the respective languages were allowed''} \citep{Hammerl1990a}.\footnote{The German original reads ``Guiraud (1959) hat aus rein kombinatorischen \"{U}berlegungen, wo alle m\"{o}glichen Buchstabenkombinationen aus den Buchstaben der jeweiligen Sprache bei der Bildung von W\"{o}rtern zugelassen wurden [...] folgende Abh\"{a}ngigkeit [...] abgeleitet.'' This is followed by the formulae given above in the main text.} Unfortunately, we have not been able to find a proper reference to Guiraud's work of 1959. Therefore, we cannot tell if Guiraud was following some optimization hypothesis akin to optimal singular-coding or if he actually provided an exact formula like ours. 
Finally, the logarithmic relationship between the frequency of a word and its length in phonemes has also been inferred based on empirical data collected for overall eight languages (see Equation 11 in \citet{Guiter1974}). However, this particular study is bare of any mathematical/information theoretic considerations.    

Besides 
historical considerations, our findings also have practical implications for empirical research on the law of abbreviation as an indication of optimal coding. First, it is usually assumed that a significant negative correlation between frequency and magnitude is needed for efficient coding \citep{Ferrer2009g,Semple2010a,Bezerra2011,Heesen2019a}. Our analyses indicate that a non-significant correlation can still be associated with efficient coding. For instance, we have seen that optimal coding with prescribed probabilities and magnitudes coming from some given multiset is equivalent to $\tau(p_i, l_i) \leq 0$ (Corollary \ref{Kendal_correlation_corollary}). The same conclusion can be reached from optimal uniquely decodable coding, where all strings must have the same length when types are equally likely (recall Equation \ref{optimal_coding_supplementary_equation}). Therefore, the influence of compression could be wider than commonly believed. 
What cannot be attributed to compression is the significant positive correlation between frequency and magnitude that has been found in a subset of the repertoire of chimpanzee gestures, i.e. full body gestures \citep{Heesen2019a}, in the vocalizations of female (but not male) hyraxes \citep{Demartsev2019a}, phrases of male gibbon solos \citep{Clink2020a}, computer experiments with neural networks \citep{Chaabouni2019a} and also in European heraldry \citep{Milton2019a}. Importantly, this illustrates that compression -- as reflected in the law of abbreviation -- is not necessarily found in all communication systems, which undermines arguments that quantitative linguistic laws are unavoidable and hence ``meaningless'' (see also \cite{Ferrer2012f} for the case of Menzerath's law).

Another argument along those lines is based on random typing: if random typing recreates Zipfian laws, then surely they are not an interesting subject of study \citep{Miller1957}. However, surprisingly, random typing turns out to be an optimal encoding system. Thus, finding linguistic laws in random typing does not preclude that these laws can be explained by information theoretic principles. However, while we have unveiled the optimality of random typing, we emphasize that we have done it only from the perspective of optimal non-singular coding. The fact that random typing and optimization are not independent issues as commonly believed \citep{Miller1957,Li1998,Kanwal2017a,Chaabouni2019a}, does not imply that random typing satisfies to a sufficient degree the optimization constraints imposed on natural languages.

We have seen that optimal non-singular coding predicts both a form of Zipf's law of abbreviation as well as a power-law distribution consistent with Zipf's law for word frequencies when combined with the maxent principle, revisiting an old argument by Mandelbrot \citep{Mandelbrot1966}. The capacity of maxent to obtain Zipfian laws as well as the less popular exponential distribution of parts-of-speech \citep{Tuzzi2010a} 
based on optimal and suboptimal coding considerations
suggests that the principle should be considered as a critical component of a compact theory of linguistic patterns in general. For instance, $p(d)$, the probability that two syntactically related words are at distance $d$ (in words), exhibits an exponential decay that has been derived with the help of a combination of maxent and a constraint on the average value of $d$ \citep{Ferrer2004b}.

The principle of maximum entropy used to derive Zipf's law for word frequencies ensures that {\em one is maximally uncertain about what one does not know} \citep{Kesavan2009a}. In the context of natural languages, a further justification of the use of the principle is that $I(S, R)$, the mutual information between words ($S$) and meanings ($R$) satisfies
\begin{equation}
I(S, R) \leq H(S),
\label{mutual_information_upper_bound_equation}
\end{equation}
where $H(S)$ is the entropy of words, namely the entropy of word probability ranks as defined above. The inequality in Equation \ref{mutual_information_upper_bound_equation} follows from elementary information theory \citep{Cover2006a}, and has been applied to investigate the properties of dual optimization models of natural communication \citep{Ferrer2007a}. $I(S, R)$ is a measure of the capacity of words to convey meaning: maximizing $I(S, R)$ one promotes that words behave like meaning identifiers \citep[Section 3]{Ferrer2015b}.  
Therefore, Equation \ref{mutual_information_upper_bound_equation} suggests that the maximum entropy principle in the context of word entropy maximizes the potential of words to express meaning. The hypothesis of pressure to maximize $H(S)$ is supported by the skew towards the right that is found in the distribution of $H(S)$ in languages across the world \citep{Bentz2017a}. 

The challenge of mathematical modelling is to find a compromise between parsimony and predictive power \citep{Burnham2002a}. Concerns about parsimony are a recurrent theme when modelling Zipf's law for word frequencies \citep{Mandelbrot1966,Visser2013a,Ferrer2015b}. As for maximum entropy models, it has been argued that Shannon entropy and a logarithmic constraint offer the simplest explanation for the origins of the law \citep{Visser2013a}. However, the argument is incomplete unless a justification for such a constraint is provided. Here we have shown how the logarithmic constraint follows from optimal non-singular coding. There are many possible explanations for the origins of Zipf's law based on maximum entropy \citep{Mandelbrot1966,Naranan1992b,Naranan1992c,Naranan1993,Ferrer2003c,Liu2008c,Baek2011a,Visser2013a}, and many more through other means \citep{Mitzenmacher2003a,Newman2004a}, but only compression can shed light on the origins of both Zipf's law for word frequencies and Zipf's law of abbreviation. The explanation of Zipf's law for word frequencies should not be separated from the explanation of other quantitative laws. Otherwise, the space of possible models is not sufficiently constrained \citep{Stumpf2012a}, and the resulting ``theory'' is not a well organized theory but a patchwork of models \citep{Ferrer2015b}. 

Our theoretical framework is highly predictive in at least two senses. First,  optimal coding predicts Zipf's law of abbreviation, but adherence to a traditional scheme (non-singular coding or uniquely decodable coding) is not necessary. It suffices to assume that the magnitudes come from some predefined multiset. Second, its applicability goes beyond laws from Zipf's classic work. It can also be applied to Menzerath's law, the tendency of constructs with more parts to be made of smaller parts, i.e. the tendency of words with more syllables to be made of shorter syllables \citep{Altmann1980a}. Taking the number of parts of constructs as probabilities of types ($p_i$'s) and the size of the parts as magnitudes ($l_i$'s) and simply assuming that the number of parts are constant, Menzerath's law follows applying theorem \ref{optimal_coding_theorem} \citep{Gustison2016a}. This allows one to put forward optimization as a possible hypothesis to explain the pervasiveness of the law in nature (e.g. \citet{Boroda1991a,Shahzad2015a,Gustison2016a}). 


\iftoggle{anonymous}{
}
{
\section*{Acknowledgements}

This article is dedicated to the memory of P. Grzybek (1957-2019), co-author of one of the works referenced \citep{Strauss2007a}.
We thank A. Hern\'andez-Fern\'andez for many corrections and valuable comments and to L. Debowski for helping us to strengthen some of the mathematical proofs in the early stages of this research project.
We also thank N. Ay and M. Gustison for helpful discussions. 
RFC is supported by the grant TIN2017-89244-R from MINECO (Ministerio de Economia, Industria y Competitividad) and the recognition 2017SGR-856
(MACDA) from AGAUR (Generalitat de Catalunya). CB is supported by the DFG Center for Advanced Studies \textit{Words, Bones, Genes, Tools} at the University of T\"{u}bingen, and by the Swiss National Foundation Grant on ``Non-randomness in Morphological Diversity: A Computational Approach Based on Multilingual Corpora'' (SNF 176305) at the University of Z\"{u}rich. 
CS is funded by a Melbourne Research Scholarship.
}

\bibliographystyle{apacite}


\bibliography{../law_of_abbreviation/biblio,Ramon}

\begin{thebibliography}{}

\bibitem [\protect \citeauthoryear {%
Aadam%
}{%
Aadam%
}{%
{\protect \APACyear {2016}}%
}]{%
Aadam2016a}
\APACinsertmetastar {%
Aadam2016a}%
\begin{APACrefauthors}%
Aadam.%
\end{APACrefauthors}%
\unskip\
\newblock
\APACrefYearMonthDay{2016}{}{}.
\newblock
{\BBOQ}\APACrefatitle {Minimum dot product} {Minimum dot product}.{\BBCQ}
\newblock
\APACjournalVolNumPages{https://medium.com/competitive/minimum-dot-product-62daa5281ba6}{}{}{}.
\PrintBackRefs{\CurrentBib}

\bibitem [\protect \citeauthoryear {%
Akmajian%
, Demers%
, Farmer%
\BCBL {}\ \BBA {} Harnish%
}{%
Akmajian%
\ \protect \BOthers {.}}{%
{\protect \APACyear {1995}}%
}]{%
Akmajian1995a}
\APACinsertmetastar {%
Akmajian1995a}%
\begin{APACrefauthors}%
Akmajian, A.%
, Demers, R\BPBI A.%
, Farmer, A\BPBI K.%
\BCBL {}\ \BBA {} Harnish, R\BPBI M.%
\end{APACrefauthors}%
\unskip\
\newblock
\APACrefYear{1995}.
\newblock
\APACrefbtitle {Linguistics. An Introduction to Language and Communication}
  {Linguistics. an introduction to language and communication}\
  (\PrintOrdinal{4th}\ \BEd).
\newblock
\APACaddressPublisher{}{MIT Press}.
\PrintBackRefs{\CurrentBib}

\bibitem [\protect \citeauthoryear {%
Altmann%
}{%
Altmann%
}{%
{\protect \APACyear {1980}}%
}]{%
Altmann1980a}
\APACinsertmetastar {%
Altmann1980a}%
\begin{APACrefauthors}%
Altmann, G.%
\end{APACrefauthors}%
\unskip\
\newblock
\APACrefYearMonthDay{1980}{}{}.
\newblock
{\BBOQ}\APACrefatitle {Prolegomena to {Menzerath's} law} {Prolegomena to
  {Menzerath's} law}.{\BBCQ}
\newblock
\APACjournalVolNumPages{Glottometrika}{2}{}{1-10}.
\PrintBackRefs{\CurrentBib}

\bibitem [\protect \citeauthoryear {%
Baek%
, Bernhardsson%
\BCBL {}\ \BBA {} Minnhagen%
}{%
Baek%
\ \protect \BOthers {.}}{%
{\protect \APACyear {2011}}%
}]{%
Baek2011a}
\APACinsertmetastar {%
Baek2011a}%
\begin{APACrefauthors}%
Baek, S\BPBI K.%
, Bernhardsson, S.%
\BCBL {}\ \BBA {} Minnhagen, P.%
\end{APACrefauthors}%
\unskip\
\newblock
\APACrefYearMonthDay{2011}{}{}.
\newblock
{\BBOQ}\APACrefatitle {Zipf's law unzipped} {Zipf's law unzipped}.{\BBCQ}
\newblock
\APACjournalVolNumPages{New Journal of Physics}{13}{4}{043004}.
\PrintBackRefs{\CurrentBib}

\bibitem [\protect \citeauthoryear {%
Bentz%
, Alikaniotis%
, Cysouw%
\BCBL {}\ \BBA {} {Ferrer-i-Cancho}%
}{%
Bentz%
\ \protect \BOthers {.}}{%
{\protect \APACyear {2017}}%
}]{%
Bentz2017a}
\APACinsertmetastar {%
Bentz2017a}%
\begin{APACrefauthors}%
Bentz, C.%
, Alikaniotis, D.%
, Cysouw, M.%
\BCBL {}\ \BBA {} {Ferrer-i-Cancho}, R.%
\end{APACrefauthors}%
\unskip\
\newblock
\APACrefYearMonthDay{2017}{}{}.
\newblock
{\BBOQ}\APACrefatitle {The Entropy of Words -- Learnability and Expressivity
  across More than 1000 Languages} {The entropy of words -- learnability and
  expressivity across more than 1000 languages}.{\BBCQ}
\newblock
\APACjournalVolNumPages{Entropy}{19}{6}{}.
\PrintBackRefs{\CurrentBib}

\bibitem [\protect \citeauthoryear {%
Bentz%
\ \BBA {} {Ferrer-i-Cancho}%
}{%
Bentz%
\ \BBA {} {Ferrer-i-Cancho}%
}{%
{\protect \APACyear {2016}}%
}]{%
Bentz2016a}
\APACinsertmetastar {%
Bentz2016a}%
\begin{APACrefauthors}%
Bentz, C.%
\BCBT {}\ \BBA {} {Ferrer-i-Cancho}, R.%
\end{APACrefauthors}%
\unskip\
\newblock
\APACrefYearMonthDay{2016}{}{}.
\newblock
{\BBOQ}\APACrefatitle {{Zipf's law of abbreviation as a language universal}}
  {{Zipf's law of abbreviation as a language universal}}.{\BBCQ}
\newblock
\BIn{} C.~Bentz, G.~J{\"a}ger\BCBL {}\ \BBA {} I.~Yanovich\ (\BEDS),
  \APACrefbtitle {{Proceedings of the Leiden Workshop on Capturing Phylogenetic
  Algorithms for Linguistics}.} {{Proceedings of the Leiden Workshop on
  Capturing Phylogenetic Algorithms for Linguistics}.}
\newblock
\APACaddressPublisher{}{University of {T\"u}bingen}.
\PrintBackRefs{\CurrentBib}

\bibitem [\protect \citeauthoryear {%
Bezerra%
, Souto%
, Radford%
\BCBL {}\ \BBA {} Jones%
}{%
Bezerra%
\ \protect \BOthers {.}}{%
{\protect \APACyear {2011}}%
}]{%
Bezerra2011}
\APACinsertmetastar {%
Bezerra2011}%
\begin{APACrefauthors}%
Bezerra, B\BPBI M.%
, Souto, A.%
, Radford, A\BPBI N.%
\BCBL {}\ \BBA {} Jones, G.%
\end{APACrefauthors}%
\unskip\
\newblock
\APACrefYearMonthDay{2011}{}{}.
\newblock
{\BBOQ}\APACrefatitle {Brevity is not always a virtue in primate communication}
  {Brevity is not always a virtue in primate communication}.{\BBCQ}
\newblock
\APACjournalVolNumPages{Biology letters}{7}{1}{23--25}.
\PrintBackRefs{\CurrentBib}

\bibitem [\protect \citeauthoryear {%
Borda%
}{%
Borda%
}{%
{\protect \APACyear {2011}}%
}]{%
Borda2011a}
\APACinsertmetastar {%
Borda2011a}%
\begin{APACrefauthors}%
Borda, M.%
\end{APACrefauthors}%
\unskip\
\newblock
\APACrefYear{2011}.
\newblock
\APACrefbtitle {Fundamentals in information theory and coding} {Fundamentals in
  information theory and coding}\ (\PrintOrdinal{1st}\ \BEd).
\newblock
\APACaddressPublisher{Berlin}{Springer}.
\PrintBackRefs{\CurrentBib}

\bibitem [\protect \citeauthoryear {%
Boroda%
\ \BBA {} Altmann%
}{%
Boroda%
\ \BBA {} Altmann%
}{%
{\protect \APACyear {1991}}%
}]{%
Boroda1991a}
\APACinsertmetastar {%
Boroda1991a}%
\begin{APACrefauthors}%
Boroda, M\BPBI G.%
\BCBT {}\ \BBA {} Altmann, G.%
\end{APACrefauthors}%
\unskip\
\newblock
\APACrefYearMonthDay{1991}{}{}.
\newblock
{\BBOQ}\APACrefatitle {Menzerath's law in musical texts} {Menzerath's law in
  musical texts}.{\BBCQ}
\newblock
\APACjournalVolNumPages{Musikometrika}{3}{}{1-13}.
\PrintBackRefs{\CurrentBib}

\bibitem [\protect \citeauthoryear {%
Burnham%
\ \BBA {} Anderson%
}{%
Burnham%
\ \BBA {} Anderson%
}{%
{\protect \APACyear {2002}}%
}]{%
Burnham2002a}
\APACinsertmetastar {%
Burnham2002a}%
\begin{APACrefauthors}%
Burnham, K\BPBI P.%
\BCBT {}\ \BBA {} Anderson, D\BPBI R.%
\end{APACrefauthors}%
\unskip\
\newblock
\APACrefYear{2002}.
\newblock
\APACrefbtitle {Model selection and multimodel inference. {A} practical
  information-theoretic approach} {Model selection and multimodel inference.
  {A} practical information-theoretic approach}\ (\PrintOrdinal{2nd}\ \BEd).
\newblock
\APACaddressPublisher{New York}{Springer}.
\PrintBackRefs{\CurrentBib}

\bibitem [\protect \citeauthoryear {%
Casas%
, Catal\`a%
, {Ferrer-i-Cancho}%
, Hern\'andez-Fern\'andez%
\BCBL {}\ \BBA {} Baixeries%
}{%
Casas%
\ \protect \BOthers {.}}{%
{\protect \APACyear {2018}}%
}]{%
Casas2019a}
\APACinsertmetastar {%
Casas2019a}%
\begin{APACrefauthors}%
Casas, B.%
, Catal\`a, N.%
, {Ferrer-i-Cancho}, R.%
, Hern\'andez-Fern\'andez, A.%
\BCBL {}\ \BBA {} Baixeries, J.%
\end{APACrefauthors}%
\unskip\
\newblock
\APACrefYearMonthDay{2018}{}{}.
\newblock
{\BBOQ}\APACrefatitle {The polysemy of the words that children learn over time}
  {The polysemy of the words that children learn over time}.{\BBCQ}
\newblock
\APACjournalVolNumPages{Interaction Studies}{19}{}{389 -- 426}.
\PrintBackRefs{\CurrentBib}

\bibitem [\protect \citeauthoryear {%
Chaabouni%
, Kharitonov%
, Dupoux%
\BCBL {}\ \BBA {} Baroni%
}{%
Chaabouni%
\ \protect \BOthers {.}}{%
{\protect \APACyear {2019}}%
}]{%
Chaabouni2019a}
\APACinsertmetastar {%
Chaabouni2019a}%
\begin{APACrefauthors}%
Chaabouni, R.%
, Kharitonov, E.%
, Dupoux, E.%
\BCBL {}\ \BBA {} Baroni, M.%
\end{APACrefauthors}%
\unskip\
\newblock
\APACrefYearMonthDay{2019}{}{}.
\newblock
{\BBOQ}\APACrefatitle {Anti-efficient encoding in emergent communication}
  {Anti-efficient encoding in emergent communication}.{\BBCQ}
\newblock
\APACjournalVolNumPages{arXiv:1905.12561}{}{}{}.
\PrintBackRefs{\CurrentBib}

\bibitem [\protect \citeauthoryear {%
Clink%
, Ahmad%
\BCBL {}\ \BBA {} Klinck%
}{%
Clink%
\ \protect \BOthers {.}}{%
{\protect \APACyear {2020}}%
}]{%
Clink2020a}
\APACinsertmetastar {%
Clink2020a}%
\begin{APACrefauthors}%
Clink, D\BPBI J.%
, Ahmad, A\BPBI H.%
\BCBL {}\ \BBA {} Klinck, H.%
\end{APACrefauthors}%
\unskip\
\newblock
\APACrefYearMonthDay{2020}{}{}.
\newblock
{\BBOQ}\APACrefatitle {Brevity is not a universal in animal communication:
  evidence for compression depends on the unit of analysis in small ape
  vocalizations} {Brevity is not a universal in animal communication: evidence
  for compression depends on the unit of analysis in small ape
  vocalizations}.{\BBCQ}
\newblock
\APACjournalVolNumPages{Royal Society Open Science}{7}{4}{200151}.
\newblock
\begin{APACrefDOI} \doi{10.1098/rsos.200151} \end{APACrefDOI}
\PrintBackRefs{\CurrentBib}

\bibitem [\protect \citeauthoryear {%
Conover%
}{%
Conover%
}{%
{\protect \APACyear {1999}}%
}]{%
Conover1999a}
\APACinsertmetastar {%
Conover1999a}%
\begin{APACrefauthors}%
Conover, W\BPBI J.%
\end{APACrefauthors}%
\unskip\
\newblock
\APACrefYear{1999}.
\newblock
\APACrefbtitle {Practical nonparametric statistics} {Practical nonparametric
  statistics}.
\newblock
\APACaddressPublisher{New York}{Wiley}.
\newblock
\APACrefnote{3rd edition}
\PrintBackRefs{\CurrentBib}

\bibitem [\protect \citeauthoryear {%
Cover%
\ \BBA {} Thomas%
}{%
Cover%
\ \BBA {} Thomas%
}{%
{\protect \APACyear {2006}}%
}]{%
Cover2006a}
\APACinsertmetastar {%
Cover2006a}%
\begin{APACrefauthors}%
Cover, T\BPBI M.%
\BCBT {}\ \BBA {} Thomas, J\BPBI A.%
\end{APACrefauthors}%
\unskip\
\newblock
\APACrefYear{2006}.
\newblock
\APACrefbtitle {Elements of information theory} {Elements of information
  theory}.
\newblock
\APACaddressPublisher{New York}{Wiley}.
\newblock
\APACrefnote{2nd edition}
\PrintBackRefs{\CurrentBib}

\bibitem [\protect \citeauthoryear {%
Demartsev%
\ \protect \BOthers {.}}{%
Demartsev%
\ \protect \BOthers {.}}{%
{\protect \APACyear {2019}}%
}]{%
Demartsev2019a}
\APACinsertmetastar {%
Demartsev2019a}%
\begin{APACrefauthors}%
Demartsev, V.%
, Gordon, N.%
, Barocas, A.%
, {Bar-Ziv}, E.%
, T.~Ilany, Y\BPBI G.%
, Ilany, A.%
\BCBL {}\ \BBA {} Geffen, E.%
\end{APACrefauthors}%
\unskip\
\newblock
\APACrefYearMonthDay{2019}{}{}.
\newblock
{\BBOQ}\APACrefatitle {The ``{Law of Brevity}'' in animal communication:
  Sex-specific signaling optimization is determined by call amplitude rather
  than duration} {The ``{Law of Brevity}'' in animal communication:
  Sex-specific signaling optimization is determined by call amplitude rather
  than duration}.{\BBCQ}
\newblock
\APACjournalVolNumPages{Evolution Letters}{3}{}{623-634}.
\PrintBackRefs{\CurrentBib}

\bibitem [\protect \citeauthoryear {%
{Elias}%
}{%
{Elias}%
}{%
{\protect \APACyear {1975}}%
}]{%
Elias1975a}
\APACinsertmetastar {%
Elias1975a}%
\begin{APACrefauthors}%
{Elias}, P.%
\end{APACrefauthors}%
\unskip\
\newblock
\APACrefYearMonthDay{1975}{}{}.
\newblock
{\BBOQ}\APACrefatitle {Universal codeword sets and representations of the
  integers} {Universal codeword sets and representations of the
  integers}.{\BBCQ}
\newblock
\APACjournalVolNumPages{IEEE Transactions on Information
  Theory}{21}{2}{194-203}.
\newblock
\begin{APACrefDOI} \doi{10.1109/TIT.1975.1055349} \end{APACrefDOI}
\PrintBackRefs{\CurrentBib}

\bibitem [\protect \citeauthoryear {%
Favaro%
\ \protect \BOthers {.}}{%
Favaro%
\ \protect \BOthers {.}}{%
{\protect \APACyear {2020}}%
}]{%
Favaro2020a}
\APACinsertmetastar {%
Favaro2020a}%
\begin{APACrefauthors}%
Favaro, L.%
, Gamba, M.%
, Cresta, E.%
, Fumagalli, E.%
, Bandoli, F.%
, Pilenga, C.%
\BDBL {}Reby, D.%
\end{APACrefauthors}%
\unskip\
\newblock
\APACrefYearMonthDay{2020}{}{}.
\newblock
{\BBOQ}\APACrefatitle {Do penguins' vocal sequences conform to linguistic
  laws?} {Do penguins' vocal sequences conform to linguistic laws?}{\BBCQ}
\newblock
\APACjournalVolNumPages{Biology Letters}{16}{2}{20190589}.
\PrintBackRefs{\CurrentBib}

\bibitem [\protect \citeauthoryear {%
{Ferrer-i-Cancho}%
}{%
{Ferrer-i-Cancho}%
}{%
{\protect \APACyear {2004}}%
}]{%
Ferrer2004b}
\APACinsertmetastar {%
Ferrer2004b}%
\begin{APACrefauthors}%
{Ferrer-i-Cancho}, R.%
\end{APACrefauthors}%
\unskip\
\newblock
\APACrefYearMonthDay{2004}{}{}.
\newblock
{\BBOQ}\APACrefatitle {{Euclidean} distance between syntactically linked words}
  {{Euclidean} distance between syntactically linked words}.{\BBCQ}
\newblock
\APACjournalVolNumPages{Physical Review E}{70}{}{056135}.
\PrintBackRefs{\CurrentBib}

\bibitem [\protect \citeauthoryear {%
{Ferrer-i-Cancho}%
}{%
{Ferrer-i-Cancho}%
}{%
{\protect \APACyear {2005}}%
{\protect \APACexlab {{\protect \BCnt {1}}}}}]{%
Ferrer2003c}
\APACinsertmetastar {%
Ferrer2003c}%
\begin{APACrefauthors}%
{Ferrer-i-Cancho}, R.%
\end{APACrefauthors}%
\unskip\
\newblock
\APACrefYearMonthDay{2005{\protect \BCnt {1}}}{}{}.
\newblock
{\BBOQ}\APACrefatitle {Decoding least effort and scaling in signal frequency
  distributions} {Decoding least effort and scaling in signal frequency
  distributions}.{\BBCQ}
\newblock
\APACjournalVolNumPages{Physica A}{345}{}{275-284}.
\newblock
\begin{APACrefDOI} \doi{10.1016/j.physa.2004.06.158} \end{APACrefDOI}
\PrintBackRefs{\CurrentBib}

\bibitem [\protect \citeauthoryear {%
{Ferrer-i-Cancho}%
}{%
{Ferrer-i-Cancho}%
}{%
{\protect \APACyear {2005}}%
{\protect \APACexlab {{\protect \BCnt {2}}}}}]{%
Ferrer2004a}
\APACinsertmetastar {%
Ferrer2004a}%
\begin{APACrefauthors}%
{Ferrer-i-Cancho}, R.%
\end{APACrefauthors}%
\unskip\
\newblock
\APACrefYearMonthDay{2005{\protect \BCnt {2}}}{}{}.
\newblock
{\BBOQ}\APACrefatitle {The variation of {Zipf's} law in human language} {The
  variation of {Zipf's} law in human language}.{\BBCQ}
\newblock
\APACjournalVolNumPages{European Physical Journal B}{44}{}{249-257}.
\PrintBackRefs{\CurrentBib}

\bibitem [\protect \citeauthoryear {%
{Ferrer-i-Cancho}%
}{%
{Ferrer-i-Cancho}%
}{%
{\protect \APACyear {2015}}%
}]{%
Ferrer2013e}
\APACinsertmetastar {%
Ferrer2013e}%
\begin{APACrefauthors}%
{Ferrer-i-Cancho}, R.%
\end{APACrefauthors}%
\unskip\
\newblock
\APACrefYearMonthDay{2015}{}{}.
\newblock
{\BBOQ}\APACrefatitle {The placement of the head that minimizes online memory:
  a complex systems approach} {The placement of the head that minimizes online
  memory: a complex systems approach}.{\BBCQ}
\newblock
\APACjournalVolNumPages{Language Dynamics and Change}{5}{}{114-137}.
\PrintBackRefs{\CurrentBib}

\bibitem [\protect \citeauthoryear {%
{Ferrer-i-Cancho}%
}{%
{Ferrer-i-Cancho}%
}{%
{\protect \APACyear {2018}}%
}]{%
Ferrer2015b}
\APACinsertmetastar {%
Ferrer2015b}%
\begin{APACrefauthors}%
{Ferrer-i-Cancho}, R.%
\end{APACrefauthors}%
\unskip\
\newblock
\APACrefYearMonthDay{2018}{}{}.
\newblock
{\BBOQ}\APACrefatitle {Optimization models of natural communication}
  {Optimization models of natural communication}.{\BBCQ}
\newblock
\APACjournalVolNumPages{Journal of Quantitative Linguistics}{25}{}{207-237}.
\PrintBackRefs{\CurrentBib}

\bibitem [\protect \citeauthoryear {%
{Ferrer-i-Cancho}%
\ \BBA {} D\'iaz-Guilera%
}{%
{Ferrer-i-Cancho}%
\ \BBA {} D\'iaz-Guilera%
}{%
{\protect \APACyear {2007}}%
}]{%
Ferrer2007a}
\APACinsertmetastar {%
Ferrer2007a}%
\begin{APACrefauthors}%
{Ferrer-i-Cancho}, R.%
\BCBT {}\ \BBA {} D\'iaz-Guilera, A.%
\end{APACrefauthors}%
\unskip\
\newblock
\APACrefYearMonthDay{2007}{}{}.
\newblock
{\BBOQ}\APACrefatitle {The global minima of the communicative energy of natural
  communication systems} {The global minima of the communicative energy of
  natural communication systems}.{\BBCQ}
\newblock
\APACjournalVolNumPages{Journal of Statistical Mechanics}{}{}{P06009}.
\PrintBackRefs{\CurrentBib}

\bibitem [\protect \citeauthoryear {%
{Ferrer-i-Cancho}%
\ \BBA {} Elvev{\aa}g%
}{%
{Ferrer-i-Cancho}%
\ \BBA {} Elvev{\aa}g%
}{%
{\protect \APACyear {2009}}%
}]{%
Ferrer2009b}
\APACinsertmetastar {%
Ferrer2009b}%
\begin{APACrefauthors}%
{Ferrer-i-Cancho}, R.%
\BCBT {}\ \BBA {} Elvev{\aa}g, B.%
\end{APACrefauthors}%
\unskip\
\newblock
\APACrefYearMonthDay{2009}{}{}.
\newblock
{\BBOQ}\APACrefatitle {Random texts do not exhibit the real {Zipf's-law-like}
  rank distribution} {Random texts do not exhibit the real {Zipf's-law-like}
  rank distribution}.{\BBCQ}
\newblock
\APACjournalVolNumPages{PLoS ONE}{5}{4}{e9411}.
\PrintBackRefs{\CurrentBib}

\bibitem [\protect \citeauthoryear {%
{Ferrer-i-Cancho}%
, Forns%
, Hern\'andez-Fern\'andez%
, Bel-Enguix%
\BCBL {}\ \BBA {} Baixeries%
}{%
{Ferrer-i-Cancho}%
, Forns%
\BCBL {}\ \protect \BOthers {.}}{%
{\protect \APACyear {2013}}%
}]{%
Ferrer2012f}
\APACinsertmetastar {%
Ferrer2012f}%
\begin{APACrefauthors}%
{Ferrer-i-Cancho}, R.%
, Forns, N.%
, Hern\'andez-Fern\'andez, A.%
, Bel-Enguix, G.%
\BCBL {}\ \BBA {} Baixeries, J.%
\end{APACrefauthors}%
\unskip\
\newblock
\APACrefYearMonthDay{2013}{}{}.
\newblock
{\BBOQ}\APACrefatitle {The challenges of statistical patterns of language: the
  case of {Menzerath}'s law in genomes} {The challenges of statistical patterns
  of language: the case of {Menzerath}'s law in genomes}.{\BBCQ}
\newblock
\APACjournalVolNumPages{Complexity}{18}{3}{11-17}.
\PrintBackRefs{\CurrentBib}

\bibitem [\protect \citeauthoryear {%
{Ferrer-i-Cancho}%
\ \BBA {} Gavald\`a%
}{%
{Ferrer-i-Cancho}%
\ \BBA {} Gavald\`a%
}{%
{\protect \APACyear {2009}}%
}]{%
Ferrer2009a}
\APACinsertmetastar {%
Ferrer2009a}%
\begin{APACrefauthors}%
{Ferrer-i-Cancho}, R.%
\BCBT {}\ \BBA {} Gavald\`a, R.%
\end{APACrefauthors}%
\unskip\
\newblock
\APACrefYearMonthDay{2009}{}{}.
\newblock
{\BBOQ}\APACrefatitle {The frequency spectrum of finite samples from the
  intermittent silence process} {The frequency spectrum of finite samples from
  the intermittent silence process}.{\BBCQ}
\newblock
\APACjournalVolNumPages{Journal of the American Association for Information
  Science and Technology}{60}{4}{837-843}.
\PrintBackRefs{\CurrentBib}

\bibitem [\protect \citeauthoryear {%
{Ferrer-i-Cancho}%
\ \BBA {} Hern\'andez-Fern\'andez%
}{%
{Ferrer-i-Cancho}%
\ \BBA {} Hern\'andez-Fern\'andez%
}{%
{\protect \APACyear {2013}}%
}]{%
Ferrer2012a}
\APACinsertmetastar {%
Ferrer2012a}%
\begin{APACrefauthors}%
{Ferrer-i-Cancho}, R.%
\BCBT {}\ \BBA {} Hern\'andez-Fern\'andez, A.%
\end{APACrefauthors}%
\unskip\
\newblock
\APACrefYearMonthDay{2013}{}{}.
\newblock
{\BBOQ}\APACrefatitle {The failure of the law of brevity in two {New World}
  primates. {Statistical} caveats} {The failure of the law of brevity in two
  {New World} primates. {Statistical} caveats}.{\BBCQ}
\newblock
\APACjournalVolNumPages{Glottotheory}{4}{1}{}.
\PrintBackRefs{\CurrentBib}

\bibitem [\protect \citeauthoryear {%
{Ferrer-i-Cancho}%
, Hern\'{a}ndez-Fern\'{a}ndez%
\BCBL {}\ \protect \BOthers {.}}{%
{Ferrer-i-Cancho}%
, Hern\'{a}ndez-Fern\'{a}ndez%
\BCBL {}\ \protect \BOthers {.}}{%
{\protect \APACyear {2013}}%
}]{%
Ferrer2012d}
\APACinsertmetastar {%
Ferrer2012d}%
\begin{APACrefauthors}%
{Ferrer-i-Cancho}, R.%
, Hern\'{a}ndez-Fern\'{a}ndez, A.%
, Lusseau, D.%
, Agoramoorthy, G.%
, Hsu, M\BPBI J.%
\BCBL {}\ \BBA {} Semple, S.%
\end{APACrefauthors}%
\unskip\
\newblock
\APACrefYearMonthDay{2013}{}{}.
\newblock
{\BBOQ}\APACrefatitle {Compression as a universal principle of animal behavior}
  {Compression as a universal principle of animal behavior}.{\BBCQ}
\newblock
\APACjournalVolNumPages{Cognitive Science}{37}{8}{1565-1578}.
\PrintBackRefs{\CurrentBib}

\bibitem [\protect \citeauthoryear {%
{Ferrer-i-Cancho}%
\ \BBA {} Lusseau%
}{%
{Ferrer-i-Cancho}%
\ \BBA {} Lusseau%
}{%
{\protect \APACyear {2009}}%
}]{%
Ferrer2009g}
\APACinsertmetastar {%
Ferrer2009g}%
\begin{APACrefauthors}%
{Ferrer-i-Cancho}, R.%
\BCBT {}\ \BBA {} Lusseau, D.%
\end{APACrefauthors}%
\unskip\
\newblock
\APACrefYearMonthDay{2009}{}{}.
\newblock
{\BBOQ}\APACrefatitle {Efficient coding in dolphin surface behavioral patterns}
  {Efficient coding in dolphin surface behavioral patterns}.{\BBCQ}
\newblock
\APACjournalVolNumPages{Complexity}{14}{5}{23-25}.
\newblock
\begin{APACrefDOI} \doi{10.1002/cplx.20266} \end{APACrefDOI}
\PrintBackRefs{\CurrentBib}

\bibitem [\protect \citeauthoryear {%
Ficken%
, Hailman%
\BCBL {}\ \BBA {} Ficken%
}{%
Ficken%
\ \protect \BOthers {.}}{%
{\protect \APACyear {1978}}%
}]{%
Ficken1978a}
\APACinsertmetastar {%
Ficken1978a}%
\begin{APACrefauthors}%
Ficken, M\BPBI S.%
, Hailman, J\BPBI P.%
\BCBL {}\ \BBA {} Ficken, R\BPBI W.%
\end{APACrefauthors}%
\unskip\
\newblock
\APACrefYearMonthDay{1978}{}{}.
\newblock
{\BBOQ}\APACrefatitle {A model of repetitive behaviour illustrated by chickadee
  calling} {A model of repetitive behaviour illustrated by chickadee
  calling}.{\BBCQ}
\newblock
\APACjournalVolNumPages{Animal Behaviour}{26}{2}{630-631}.
\PrintBackRefs{\CurrentBib}

\bibitem [\protect \citeauthoryear {%
Guiter%
}{%
Guiter%
}{%
{\protect \APACyear {1974}}%
}]{%
Guiter1974}
\APACinsertmetastar {%
Guiter1974}%
\begin{APACrefauthors}%
Guiter, H.%
\end{APACrefauthors}%
\unskip\
\newblock
\APACrefYearMonthDay{1974}{}{}.
\newblock
{\BBOQ}\APACrefatitle {{Les relationes frequence - longueur - sens des mots
  (langues romanes et anglais)}} {{Les relationes frequence - longueur - sens
  des mots (langues romanes et anglais)}}.{\BBCQ}
\newblock
\BIn{} \APACrefbtitle {{XIV Congresso Internazionale di linguistica e filologia
  romanza}} {{XIV Congresso Internazionale di linguistica e filologia
  romanza}}\ (\BPG~373-381).
\newblock
\APACaddressPublisher{Napoli}{}.
\PrintBackRefs{\CurrentBib}

\bibitem [\protect \citeauthoryear {%
Gustison%
, Semple%
, {Ferrer-i-Cancho}%
\BCBL {}\ \BBA {} Bergman%
}{%
Gustison%
\ \protect \BOthers {.}}{%
{\protect \APACyear {2016}}%
}]{%
Gustison2016a}
\APACinsertmetastar {%
Gustison2016a}%
\begin{APACrefauthors}%
Gustison, M\BPBI L.%
, Semple, S.%
, {Ferrer-i-Cancho}, R.%
\BCBL {}\ \BBA {} Bergman, T.%
\end{APACrefauthors}%
\unskip\
\newblock
\APACrefYearMonthDay{2016}{}{}.
\newblock
{\BBOQ}\APACrefatitle {Gelada vocal sequences follow {Menzerath}'s linguistic
  law} {Gelada vocal sequences follow {Menzerath}'s linguistic law}.{\BBCQ}
\newblock
\APACjournalVolNumPages{Proceedings of the National Academy of Sciences
  USA}{113}{}{E2750-E2758}.
\newblock
\begin{APACrefDOI} \doi{doi/10.1073/pnas.1522072113} \end{APACrefDOI}
\PrintBackRefs{\CurrentBib}

\bibitem [\protect \citeauthoryear {%
Hailman%
, Ficken%
\BCBL {}\ \BBA {} Ficken%
}{%
Hailman%
\ \protect \BOthers {.}}{%
{\protect \APACyear {1985}}%
}]{%
Hailman1985}
\APACinsertmetastar {%
Hailman1985}%
\begin{APACrefauthors}%
Hailman, J\BPBI P.%
, Ficken, M\BPBI S.%
\BCBL {}\ \BBA {} Ficken, R\BPBI W.%
\end{APACrefauthors}%
\unskip\
\newblock
\APACrefYearMonthDay{1985}{}{}.
\newblock
{\BBOQ}\APACrefatitle {The 'chick-a-dee' calls of {{\em Parus}} {\em
  atricapillus}: a recombinant system of animal communication compared with
  written {English}} {The 'chick-a-dee' calls of {{\em Parus}} {\em
  atricapillus}: a recombinant system of animal communication compared with
  written {English}}.{\BBCQ}
\newblock
\APACjournalVolNumPages{Semiotica}{56}{}{121-224}.
\PrintBackRefs{\CurrentBib}

\bibitem [\protect \citeauthoryear {%
Hammerl%
}{%
Hammerl%
}{%
{\protect \APACyear {1990}}%
}]{%
Hammerl1990a}
\APACinsertmetastar {%
Hammerl1990a}%
\begin{APACrefauthors}%
Hammerl, R.%
\end{APACrefauthors}%
\unskip\
\newblock
\APACrefYearMonthDay{1990}{}{}.
\newblock
{\BBOQ}\APACrefatitle {{L{\"a}nge} - {Frequenz}, {L\"ange} - {Rangnummer}:
  {\"Uberpr\"ufung} von zwei lexikalischen {Modellen.}} {{L{\"a}nge} -
  {Frequenz}, {L\"ange} - {Rangnummer}: {\"Uberpr\"ufung} von zwei
  lexikalischen {Modellen.}}{\BBCQ}
\newblock
\APACjournalVolNumPages{Glottometrika}{12}{}{1--24}.
\PrintBackRefs{\CurrentBib}

\bibitem [\protect \citeauthoryear {%
Harremo{\"e}s%
\ \BBA {} Tops{\o}e%
}{%
Harremo{\"e}s%
\ \BBA {} Tops{\o}e%
}{%
{\protect \APACyear {2001}}%
}]{%
Harremoes2001a}
\APACinsertmetastar {%
Harremoes2001a}%
\begin{APACrefauthors}%
Harremo{\"e}s, P.%
\BCBT {}\ \BBA {} Tops{\o}e, F.%
\end{APACrefauthors}%
\unskip\
\newblock
\APACrefYearMonthDay{2001}{}{}.
\newblock
{\BBOQ}\APACrefatitle {Maximum Entropy Fundamentals} {Maximum entropy
  fundamentals}.{\BBCQ}
\newblock
\APACjournalVolNumPages{Entropy}{3}{3}{191--226}.
\newblock
\begin{APACrefDOI} \doi{10.3390/e3030191} \end{APACrefDOI}
\PrintBackRefs{\CurrentBib}

\bibitem [\protect \citeauthoryear {%
Heesen%
, Hobaiter%
, {Ferrer-i-Cancho}%
\BCBL {}\ \BBA {} Semple%
}{%
Heesen%
\ \protect \BOthers {.}}{%
{\protect \APACyear {2019}}%
}]{%
Heesen2019a}
\APACinsertmetastar {%
Heesen2019a}%
\begin{APACrefauthors}%
Heesen, R.%
, Hobaiter, C.%
, {Ferrer-i-Cancho}, R.%
\BCBL {}\ \BBA {} Semple, S.%
\end{APACrefauthors}%
\unskip\
\newblock
\APACrefYearMonthDay{2019}{}{}.
\newblock
{\BBOQ}\APACrefatitle {Linguistic laws in chimpanzee gestural communication}
  {Linguistic laws in chimpanzee gestural communication}.{\BBCQ}
\newblock
\APACjournalVolNumPages{Proceedings of the Royal Society B: Biological
  Sciences}{286}{}{20182900}.
\PrintBackRefs{\CurrentBib}

\bibitem [\protect \citeauthoryear {%
Huang%
, Ma%
, Ma%
, Garber%
\BCBL {}\ \BBA {} Fan%
}{%
Huang%
\ \protect \BOthers {.}}{%
{\protect \APACyear {2020}}%
}]{%
Huang2020a}
\APACinsertmetastar {%
Huang2020a}%
\begin{APACrefauthors}%
Huang, M.%
, Ma, H.%
, Ma, C.%
, Garber, P\BPBI A.%
\BCBL {}\ \BBA {} Fan, P.%
\end{APACrefauthors}%
\unskip\
\newblock
\APACrefYearMonthDay{2020}{}{}.
\newblock
{\BBOQ}\APACrefatitle {Male gibbon loud morning calls conform to {Zipf's} law
  of brevity and {Menzerath's} law: insights into the origin of human language}
  {Male gibbon loud morning calls conform to {Zipf's} law of brevity and
  {Menzerath's} law: insights into the origin of human language}.{\BBCQ}
\newblock
\APACjournalVolNumPages{Animal Behaviour}{160}{}{145 - 155}.
\PrintBackRefs{\CurrentBib}

\bibitem [\protect \citeauthoryear {%
Kanwal%
, Smith%
, Culbertson%
\BCBL {}\ \BBA {} Kirby%
}{%
Kanwal%
\ \protect \BOthers {.}}{%
{\protect \APACyear {2017}}%
}]{%
Kanwal2017a}
\APACinsertmetastar {%
Kanwal2017a}%
\begin{APACrefauthors}%
Kanwal, J.%
, Smith, K.%
, Culbertson, J.%
\BCBL {}\ \BBA {} Kirby, S.%
\end{APACrefauthors}%
\unskip\
\newblock
\APACrefYearMonthDay{2017}{}{}.
\newblock
{\BBOQ}\APACrefatitle {Zipf's Law of Abbreviation and the Principle of Least
  Effort: Language users optimise a miniature lexicon for efficient
  communication} {Zipf's law of abbreviation and the principle of least effort:
  Language users optimise a miniature lexicon for efficient
  communication}.{\BBCQ}
\newblock
\APACjournalVolNumPages{Cognition}{165}{}{45-52}.
\PrintBackRefs{\CurrentBib}

\bibitem [\protect \citeauthoryear {%
Kapur%
\ \BBA {} Kesavan%
}{%
Kapur%
\ \BBA {} Kesavan%
}{%
{\protect \APACyear {1992}}%
}]{%
Kapur1992a}
\APACinsertmetastar {%
Kapur1992a}%
\begin{APACrefauthors}%
Kapur, J\BPBI N.%
\BCBT {}\ \BBA {} Kesavan, H\BPBI K.%
\end{APACrefauthors}%
\unskip\
\newblock
\APACrefYearMonthDay{1992}{}{}.
\newblock
{\BBOQ}\APACrefatitle {Entropy Optimization Principles and Their Applications}
  {Entropy optimization principles and their applications}.{\BBCQ}
\newblock
\BIn{} V\BPBI P.~Singh\ \BBA {} M.~Fiorentino\ (\BEDS), \APACrefbtitle {Entropy
  and Energy Dissipation in Water Resources} {Entropy and energy dissipation in
  water resources}\ (\BPGS\ 3--20).
\newblock
\APACaddressPublisher{Dordrecht}{Springer Netherlands}.
\PrintBackRefs{\CurrentBib}

\bibitem [\protect \citeauthoryear {%
Kesavan%
}{%
Kesavan%
}{%
{\protect \APACyear {2009}}%
}]{%
Kesavan2009a}
\APACinsertmetastar {%
Kesavan2009a}%
\begin{APACrefauthors}%
Kesavan, H\BPBI K.%
\end{APACrefauthors}%
\unskip\
\newblock
\APACrefYearMonthDay{2009}{}{}.
\newblock
{\BBOQ}\APACrefatitle {Jaynes' maximum entropy principle} {Jaynes' maximum
  entropy principle}.{\BBCQ}
\newblock
\BIn{} C\BPBI A.~Floudas\ \BBA {} P\BPBI M.~Pardalos\ (\BEDS), \APACrefbtitle
  {Encyclopedia of Optimization} {Encyclopedia of optimization}\ (\BPGS\
  1779--1782).
\newblock
\APACaddressPublisher{Boston, MA}{Springer US}.
\PrintBackRefs{\CurrentBib}

\bibitem [\protect \citeauthoryear {%
Li%
}{%
Li%
}{%
{\protect \APACyear {1992}}%
}]{%
Li1992b}
\APACinsertmetastar {%
Li1992b}%
\begin{APACrefauthors}%
Li, W.%
\end{APACrefauthors}%
\unskip\
\newblock
\APACrefYearMonthDay{1992}{}{}.
\newblock
{\BBOQ}\APACrefatitle {Random Texts Exhibit {Z}ipf's-Law-Like Word Frequency
  Distribution} {Random texts exhibit {Z}ipf's-law-like word frequency
  distribution}.{\BBCQ}
\newblock
\APACjournalVolNumPages{IEEE T. Inform. Theory}{38}{6}{1842-1845}.
\PrintBackRefs{\CurrentBib}

\bibitem [\protect \citeauthoryear {%
Li%
}{%
Li%
}{%
{\protect \APACyear {1998}}%
}]{%
Li1998}
\APACinsertmetastar {%
Li1998}%
\begin{APACrefauthors}%
Li, W.%
\end{APACrefauthors}%
\unskip\
\newblock
\APACrefYearMonthDay{1998}{}{}.
\newblock
{\BBOQ}\APACrefatitle {Comments to "{Z}ipf's Law and the structure and
  evolution of languages" {A.A. Tsonis, C. Schultz, P.A. Tsonis}, {Complexity},
  2(5). 12-13 (1997)} {Comments to "{Z}ipf's law and the structure and
  evolution of languages" {A.A. Tsonis, C. Schultz, P.A. Tsonis}, {Complexity},
  2(5). 12-13 (1997)}.{\BBCQ}
\newblock
\APACjournalVolNumPages{Complexity}{3}{}{9-10}.
\newblock
\APACrefnote{Letters to the editor}
\PrintBackRefs{\CurrentBib}

\bibitem [\protect \citeauthoryear {%
Liu%
}{%
Liu%
}{%
{\protect \APACyear {2008}}%
}]{%
Liu2008c}
\APACinsertmetastar {%
Liu2008c}%
\begin{APACrefauthors}%
Liu, C\BHBI S.%
\end{APACrefauthors}%
\unskip\
\newblock
\APACrefYearMonthDay{2008}{}{}.
\newblock
{\BBOQ}\APACrefatitle {Maximal non-symmetric entropy leads naturally to
  {Zipf's} law} {Maximal non-symmetric entropy leads naturally to {Zipf's}
  law}.{\BBCQ}
\newblock
\APACjournalVolNumPages{Fractals}{16}{01}{99-101}.
\PrintBackRefs{\CurrentBib}

\bibitem [\protect \citeauthoryear {%
Luo%
\ \protect \BOthers {.}}{%
Luo%
\ \protect \BOthers {.}}{%
{\protect \APACyear {2013}}%
}]{%
Luo2013a}
\APACinsertmetastar {%
Luo2013a}%
\begin{APACrefauthors}%
Luo, B.%
, Jiang, T.%
, Liu, Y.%
, Wang, J.%
, Lin, A.%
, Wei, X.%
\BCBL {}\ \BBA {} Feng, J.%
\end{APACrefauthors}%
\unskip\
\newblock
\APACrefYearMonthDay{2013}{}{}.
\newblock
{\BBOQ}\APACrefatitle {Brevity is prevalent in bat short-range communication}
  {Brevity is prevalent in bat short-range communication}.{\BBCQ}
\newblock
\APACjournalVolNumPages{Journal of Comparative Physiology A}{199}{}{325-333}.
\PrintBackRefs{\CurrentBib}

\bibitem [\protect \citeauthoryear {%
Mandelbrot%
}{%
Mandelbrot%
}{%
{\protect \APACyear {1966}}%
}]{%
Mandelbrot1966}
\APACinsertmetastar {%
Mandelbrot1966}%
\begin{APACrefauthors}%
Mandelbrot, B.%
\end{APACrefauthors}%
\unskip\
\newblock
\APACrefYearMonthDay{1966}{}{}.
\newblock
{\BBOQ}\APACrefatitle {Information theory and psycholinguistics: a theory of
  word frequencies} {Information theory and psycholinguistics: a theory of word
  frequencies}.{\BBCQ}
\newblock
\BIn{} P\BPBI F.~Lazarsfield\ \BBA {} N\BPBI W.~Henry\ (\BEDS), \APACrefbtitle
  {Readings in mathematical social sciences} {Readings in mathematical social
  sciences}\ (\BPG~151-168).
\newblock
\APACaddressPublisher{Cambridge}{MIT Press}.
\PrintBackRefs{\CurrentBib}

\bibitem [\protect \citeauthoryear {%
{McMillan}%
}{%
{McMillan}%
}{%
{\protect \APACyear {1956}}%
}]{%
McMillan1956a}
\APACinsertmetastar {%
McMillan1956a}%
\begin{APACrefauthors}%
{McMillan}, B.%
\end{APACrefauthors}%
\unskip\
\newblock
\APACrefYearMonthDay{1956}{}{}.
\newblock
{\BBOQ}\APACrefatitle {Two inequalities implied by unique decipherability} {Two
  inequalities implied by unique decipherability}.{\BBCQ}
\newblock
\APACjournalVolNumPages{IRE Transactions on Information Theory}{2}{4}{115-116}.
\newblock
\begin{APACrefDOI} \doi{10.1109/TIT.1956.1056818} \end{APACrefDOI}
\PrintBackRefs{\CurrentBib}

\bibitem [\protect \citeauthoryear {%
Miller%
}{%
Miller%
}{%
{\protect \APACyear {1957}}%
}]{%
Miller1957}
\APACinsertmetastar {%
Miller1957}%
\begin{APACrefauthors}%
Miller, G\BPBI A.%
\end{APACrefauthors}%
\unskip\
\newblock
\APACrefYearMonthDay{1957}{}{}.
\newblock
{\BBOQ}\APACrefatitle {Some effects of intermittent silence} {Some effects of
  intermittent silence}.{\BBCQ}
\newblock
\APACjournalVolNumPages{Am. J. Psychol.}{70}{}{311-314}.
\PrintBackRefs{\CurrentBib}

\bibitem [\protect \citeauthoryear {%
Miller%
\ \BBA {} Chomsky%
}{%
Miller%
\ \BBA {} Chomsky%
}{%
{\protect \APACyear {1963}}%
}]{%
Miller1963}
\APACinsertmetastar {%
Miller1963}%
\begin{APACrefauthors}%
Miller, G\BPBI A.%
\BCBT {}\ \BBA {} Chomsky, N.%
\end{APACrefauthors}%
\unskip\
\newblock
\APACrefYearMonthDay{1963}{}{}.
\newblock
{\BBOQ}\APACrefatitle {Finitary models of language users} {Finitary models of
  language users}.{\BBCQ}
\newblock
\BIn{} R\BPBI D.~Luce, R.~Bush\BCBL {}\ \BBA {} E.~Galanter\ (\BEDS),
  \APACrefbtitle {Handbook of Mathematical Psychology} {Handbook of
  mathematical psychology}\ (\BVOL~2, \BPG~419-491).
\newblock
\APACaddressPublisher{New York}{Wiley}.
\PrintBackRefs{\CurrentBib}

\bibitem [\protect \citeauthoryear {%
Miton%
\ \BBA {} Morin%
}{%
Miton%
\ \BBA {} Morin%
}{%
{\protect \APACyear {2019}}%
}]{%
Milton2019a}
\APACinsertmetastar {%
Milton2019a}%
\begin{APACrefauthors}%
Miton, H.%
\BCBT {}\ \BBA {} Morin, O.%
\end{APACrefauthors}%
\unskip\
\newblock
\APACrefYearMonthDay{2019}{}{}.
\newblock
{\BBOQ}\APACrefatitle {When iconicity stands in the way of abbreviation: {No
  Zipfian} effect for figurative signals} {When iconicity stands in the way of
  abbreviation: {No Zipfian} effect for figurative signals}.{\BBCQ}
\newblock
\APACjournalVolNumPages{PLOS ONE}{14}{8}{1-19}.
\PrintBackRefs{\CurrentBib}

\bibitem [\protect \citeauthoryear {%
Mitzenmacher%
}{%
Mitzenmacher%
}{%
{\protect \APACyear {2003}}%
}]{%
Mitzenmacher2003a}
\APACinsertmetastar {%
Mitzenmacher2003a}%
\begin{APACrefauthors}%
Mitzenmacher, M.%
\end{APACrefauthors}%
\unskip\
\newblock
\APACrefYearMonthDay{2003}{}{}.
\newblock
{\BBOQ}\APACrefatitle {A brief history of generative models for power law and
  lognormal distributions} {A brief history of generative models for power law
  and lognormal distributions}.{\BBCQ}
\newblock
\APACjournalVolNumPages{Internet Mathematics}{1}{}{226-251}.
\PrintBackRefs{\CurrentBib}

\bibitem [\protect \citeauthoryear {%
Moreno-S\'anchez%
, Font-Clos%
\BCBL {}\ \BBA {} Corral%
}{%
Moreno-S\'anchez%
\ \protect \BOthers {.}}{%
{\protect \APACyear {2016}}%
}]{%
Moreno2016a}
\APACinsertmetastar {%
Moreno2016a}%
\begin{APACrefauthors}%
Moreno-S\'anchez, I.%
, Font-Clos, F.%
\BCBL {}\ \BBA {} Corral, A.%
\end{APACrefauthors}%
\unskip\
\newblock
\APACrefYearMonthDay{2016}{}{}.
\newblock
{\BBOQ}\APACrefatitle {Large-scale analysis of {Zipf's} law in {English}}
  {Large-scale analysis of {Zipf's} law in {English}}.{\BBCQ}
\newblock
\APACjournalVolNumPages{PLoS ONE}{11}{}{1-19}.
\newblock
\begin{APACrefDOI} \doi{10.1371/journal.pone.0147073} \end{APACrefDOI}
\PrintBackRefs{\CurrentBib}

\bibitem [\protect \citeauthoryear {%
Naranan%
\ \BBA {} Balasubrahmanyan%
}{%
Naranan%
\ \BBA {} Balasubrahmanyan%
}{%
{\protect \APACyear {1992}}%
{\protect \APACexlab {{\protect \BCnt {1}}}}}]{%
Naranan1992b}
\APACinsertmetastar {%
Naranan1992b}%
\begin{APACrefauthors}%
Naranan, S.%
\BCBT {}\ \BBA {} Balasubrahmanyan, V\BPBI K.%
\end{APACrefauthors}%
\unskip\
\newblock
\APACrefYearMonthDay{1992{\protect \BCnt {1}}}{}{}.
\newblock
{\BBOQ}\APACrefatitle {Information theoretic models in statistical linguistics
  - {Part I: A} model for word frequencies} {Information theoretic models in
  statistical linguistics - {Part I: A} model for word frequencies}.{\BBCQ}
\newblock
\APACjournalVolNumPages{Current Science}{63}{}{261-269}.
\PrintBackRefs{\CurrentBib}

\bibitem [\protect \citeauthoryear {%
Naranan%
\ \BBA {} Balasubrahmanyan%
}{%
Naranan%
\ \BBA {} Balasubrahmanyan%
}{%
{\protect \APACyear {1992}}%
{\protect \APACexlab {{\protect \BCnt {2}}}}}]{%
Naranan1992c}
\APACinsertmetastar {%
Naranan1992c}%
\begin{APACrefauthors}%
Naranan, S.%
\BCBT {}\ \BBA {} Balasubrahmanyan, V\BPBI K.%
\end{APACrefauthors}%
\unskip\
\newblock
\APACrefYearMonthDay{1992{\protect \BCnt {2}}}{}{}.
\newblock
{\BBOQ}\APACrefatitle {Information theoretic models in statistical linguistics
  - {Part II: Word} frequencies and hierarchical structure in language.}
  {Information theoretic models in statistical linguistics - {Part II: Word}
  frequencies and hierarchical structure in language.}{\BBCQ}
\newblock
\APACjournalVolNumPages{Current Science}{63}{}{297-306}.
\PrintBackRefs{\CurrentBib}

\bibitem [\protect \citeauthoryear {%
Naranan%
\ \BBA {} Balasubrahmanyan%
}{%
Naranan%
\ \BBA {} Balasubrahmanyan%
}{%
{\protect \APACyear {1993}}%
}]{%
Naranan1993}
\APACinsertmetastar {%
Naranan1993}%
\begin{APACrefauthors}%
Naranan, S.%
\BCBT {}\ \BBA {} Balasubrahmanyan, V\BPBI K.%
\end{APACrefauthors}%
\unskip\
\newblock
\APACrefYearMonthDay{1993}{}{}.
\newblock
{\BBOQ}\APACrefatitle {Information theoretic model for frequency distribution
  of words and speech sounds (phonemes) in language} {Information theoretic
  model for frequency distribution of words and speech sounds (phonemes) in
  language}.{\BBCQ}
\newblock
\APACjournalVolNumPages{Journal of Scientific and Industrial
  Research}{52}{}{728-738}.
\PrintBackRefs{\CurrentBib}

\bibitem [\protect \citeauthoryear {%
Newman%
}{%
Newman%
}{%
{\protect \APACyear {2005}}%
}]{%
Newman2004a}
\APACinsertmetastar {%
Newman2004a}%
\begin{APACrefauthors}%
Newman, M\BPBI E\BPBI J.%
\end{APACrefauthors}%
\unskip\
\newblock
\APACrefYearMonthDay{2005}{}{}.
\newblock
{\BBOQ}\APACrefatitle {Power laws, {Pareto} distributions and {Zipf's} law}
  {Power laws, {Pareto} distributions and {Zipf's} law}.{\BBCQ}
\newblock
\APACjournalVolNumPages{Contemporary Physics}{46}{}{323-351}.
\PrintBackRefs{\CurrentBib}

\bibitem [\protect \citeauthoryear {%
Piantadosi%
, Tilly%
\BCBL {}\ \BBA {} Gibson%
}{%
Piantadosi%
\ \protect \BOthers {.}}{%
{\protect \APACyear {2012}}%
}]{%
Piantadosi2012a}
\APACinsertmetastar {%
Piantadosi2012a}%
\begin{APACrefauthors}%
Piantadosi, S\BPBI T.%
, Tilly, H.%
\BCBL {}\ \BBA {} Gibson, E.%
\end{APACrefauthors}%
\unskip\
\newblock
\APACrefYearMonthDay{2012}{}{}.
\newblock
{\BBOQ}\APACrefatitle {The communicative function of ambiguity in language}
  {The communicative function of ambiguity in language}.{\BBCQ}
\newblock
\APACjournalVolNumPages{Cognition}{122}{3}{280 - 291}.
\PrintBackRefs{\CurrentBib}

\bibitem [\protect \citeauthoryear {%
Ramscar%
}{%
Ramscar%
}{%
{\protect \APACyear {2019}}%
}]{%
Ramscar2019a}
\APACinsertmetastar {%
Ramscar2019a}%
\begin{APACrefauthors}%
Ramscar, M.%
\end{APACrefauthors}%
\unskip\
\newblock
\APACrefYearMonthDay{2019}{}{}.
\newblock
{\BBOQ}\APACrefatitle {Source codes in human communication} {Source codes in
  human communication}.{\BBCQ}
\newblock
\APACjournalVolNumPages{https://psyarxiv.com/e3hps}{}{}{}.
\newblock
\begin{APACrefDOI} \doi{10.31234/osf.io/e3hps} \end{APACrefDOI}
\PrintBackRefs{\CurrentBib}

\bibitem [\protect \citeauthoryear {%
Rapoport%
}{%
Rapoport%
}{%
{\protect \APACyear {1982}}%
}]{%
Rapoport1982}
\APACinsertmetastar {%
Rapoport1982}%
\begin{APACrefauthors}%
Rapoport, A.%
\end{APACrefauthors}%
\unskip\
\newblock
\APACrefYearMonthDay{1982}{}{}.
\newblock
{\BBOQ}\APACrefatitle {Zipf's law re-visited} {Zipf's law re-visited}.{\BBCQ}
\newblock
\BIn{} H.~Guiter\ \BBA {} M\BPBI V.~Arapov\ (\BEDS), \APACrefbtitle
  {Quantitative linguistis: {Studies} on {Zipf's} law} {Quantitative
  linguistis: {Studies} on {Zipf's} law}\ (\BPG~1-28).
\newblock
\APACaddressPublisher{Bochum}{Studienverlag Dr. N. Brockmeyer}.
\PrintBackRefs{\CurrentBib}

\bibitem [\protect \citeauthoryear {%
Romberg%
\ \BBA {} Saffran%
}{%
Romberg%
\ \BBA {} Saffran%
}{%
{\protect \APACyear {2010}}%
}]{%
Romberg2010a}
\APACinsertmetastar {%
Romberg2010a}%
\begin{APACrefauthors}%
Romberg, A\BPBI R.%
\BCBT {}\ \BBA {} Saffran, J\BPBI R.%
\end{APACrefauthors}%
\unskip\
\newblock
\APACrefYearMonthDay{2010}{}{}.
\newblock
{\BBOQ}\APACrefatitle {Statistical learning and language acquisition}
  {Statistical learning and language acquisition}.{\BBCQ}
\newblock
\APACjournalVolNumPages{Wiley Interdisciplinary Reviews: Cognitive
  Science}{1}{6}{906-914}.
\newblock
\begin{APACrefDOI} \doi{10.1002/wcs.78} \end{APACrefDOI}
\PrintBackRefs{\CurrentBib}

\bibitem [\protect \citeauthoryear {%
Semple%
, Hsu%
\BCBL {}\ \BBA {} Agoramoorthy%
}{%
Semple%
\ \protect \BOthers {.}}{%
{\protect \APACyear {2010}}%
}]{%
Semple2010a}
\APACinsertmetastar {%
Semple2010a}%
\begin{APACrefauthors}%
Semple, S.%
, Hsu, M\BPBI J.%
\BCBL {}\ \BBA {} Agoramoorthy, G.%
\end{APACrefauthors}%
\unskip\
\newblock
\APACrefYearMonthDay{2010}{}{}.
\newblock
{\BBOQ}\APACrefatitle {Efficiency of coding in macaque vocal communication}
  {Efficiency of coding in macaque vocal communication}.{\BBCQ}
\newblock
\APACjournalVolNumPages{Biology Letters}{6}{}{469-471}.
\PrintBackRefs{\CurrentBib}

\bibitem [\protect \citeauthoryear {%
Shahzad%
, Mittenthal%
\BCBL {}\ \BBA {} Caetano-Anoll\'es%
}{%
Shahzad%
\ \protect \BOthers {.}}{%
{\protect \APACyear {2015}}%
}]{%
Shahzad2015a}
\APACinsertmetastar {%
Shahzad2015a}%
\begin{APACrefauthors}%
Shahzad, K.%
, Mittenthal, J.%
\BCBL {}\ \BBA {} Caetano-Anoll\'es, G.%
\end{APACrefauthors}%
\unskip\
\newblock
\APACrefYearMonthDay{2015}{}{}.
\newblock
{\BBOQ}\APACrefatitle {The organization of domains in proteins obeys
  {Menzerath-Altmann}'s law of language} {The organization of domains in
  proteins obeys {Menzerath-Altmann}'s law of language}.{\BBCQ}
\newblock
\APACjournalVolNumPages{BMC Systems Biology}{9}{}{1-13}.
\PrintBackRefs{\CurrentBib}

\bibitem [\protect \citeauthoryear {%
Sigurd%
, Eeg-Olofsson%
\BCBL {}\ \BBA {} {van Weijer}%
}{%
Sigurd%
\ \protect \BOthers {.}}{%
{\protect \APACyear {2004}}%
}]{%
Sigurd2004a}
\APACinsertmetastar {%
Sigurd2004a}%
\begin{APACrefauthors}%
Sigurd, B.%
, Eeg-Olofsson, M.%
\BCBL {}\ \BBA {} {van Weijer}, J.%
\end{APACrefauthors}%
\unskip\
\newblock
\APACrefYearMonthDay{2004}{}{}.
\newblock
{\BBOQ}\APACrefatitle {Word length, sentence length and frequency - {Zipf}
  revisited} {Word length, sentence length and frequency - {Zipf}
  revisited}.{\BBCQ}
\newblock
\APACjournalVolNumPages{Studia Linguistica}{58}{1}{37-52}.
\PrintBackRefs{\CurrentBib}

\bibitem [\protect \citeauthoryear {%
Strauss%
, Grzybek%
\BCBL {}\ \BBA {} Altmann%
}{%
Strauss%
\ \protect \BOthers {.}}{%
{\protect \APACyear {2007}}%
}]{%
Strauss2007a}
\APACinsertmetastar {%
Strauss2007a}%
\begin{APACrefauthors}%
Strauss, U.%
, Grzybek, P.%
\BCBL {}\ \BBA {} Altmann, G.%
\end{APACrefauthors}%
\unskip\
\newblock
\APACrefYearMonthDay{2007}{}{}.
\newblock
{\BBOQ}\APACrefatitle {Word length and word frequency} {Word length and word
  frequency}.{\BBCQ}
\newblock
\BIn{} P.~Grzybek\ (\BED), \APACrefbtitle {Contributions to the science of text
  and language} {Contributions to the science of text and language}\
  (\BPG~277-294).
\newblock
\APACaddressPublisher{Dordrecht}{Springer}.
\PrintBackRefs{\CurrentBib}

\bibitem [\protect \citeauthoryear {%
Stumpf%
\ \BBA {} Porter%
}{%
Stumpf%
\ \BBA {} Porter%
}{%
{\protect \APACyear {2012}}%
}]{%
Stumpf2012a}
\APACinsertmetastar {%
Stumpf2012a}%
\begin{APACrefauthors}%
Stumpf, M\BPBI P\BPBI H.%
\BCBT {}\ \BBA {} Porter, M\BPBI A.%
\end{APACrefauthors}%
\unskip\
\newblock
\APACrefYearMonthDay{2012}{}{}.
\newblock
{\BBOQ}\APACrefatitle {Critical truths about power laws} {Critical truths about
  power laws}.{\BBCQ}
\newblock
\APACjournalVolNumPages{Science}{335}{6069}{665-666}.
\newblock
\begin{APACrefDOI} \doi{10.1126/science.1216142} \end{APACrefDOI}
\PrintBackRefs{\CurrentBib}

\bibitem [\protect \citeauthoryear {%
Sudan%
}{%
Sudan%
}{%
{\protect \APACyear {2006}}%
}]{%
Sudan2006a}
\APACinsertmetastar {%
Sudan2006a}%
\begin{APACrefauthors}%
Sudan, M.%
\end{APACrefauthors}%
\unskip\
\newblock
\APACrefYearMonthDay{2006}{}{}.
\newblock
{\BBOQ}\APACrefatitle {Lecture 07 (03/02): Source Coding: Shannon codes,
  Huffman codes. Transmission of information} {Lecture 07 (03/02): Source
  coding: Shannon codes, huffman codes. transmission of information}.{\BBCQ}
\newblock
\APACjournalVolNumPages{\url{http://people.csail.mit.edu/madhu/ST06/scribe/L07_xshi_main.pdf}}{}{}{}.
\PrintBackRefs{\CurrentBib}

\bibitem [\protect \citeauthoryear {%
Suzuki%
, Tyack%
\BCBL {}\ \BBA {} Buck%
}{%
Suzuki%
\ \protect \BOthers {.}}{%
{\protect \APACyear {2005}}%
}]{%
Suzuki2004a}
\APACinsertmetastar {%
Suzuki2004a}%
\begin{APACrefauthors}%
Suzuki, R.%
, Tyack, P\BPBI L.%
\BCBL {}\ \BBA {} Buck, J.%
\end{APACrefauthors}%
\unskip\
\newblock
\APACrefYearMonthDay{2005}{}{}.
\newblock
{\BBOQ}\APACrefatitle {The use of {Zipf's} law in animal communication
  analysis} {The use of {Zipf's} law in animal communication analysis}.{\BBCQ}
\newblock
\APACjournalVolNumPages{Anim. Behav.}{69}{}{9-17}.
\PrintBackRefs{\CurrentBib}

\bibitem [\protect \citeauthoryear {%
Tuzzi%
, Popescu%
\BCBL {}\ \BBA {} Altmann%
}{%
Tuzzi%
\ \protect \BOthers {.}}{%
{\protect \APACyear {2010}}%
}]{%
Tuzzi2010a}
\APACinsertmetastar {%
Tuzzi2010a}%
\begin{APACrefauthors}%
Tuzzi, A.%
, Popescu, I\BHBI I.%
\BCBL {}\ \BBA {} Altmann, G.%
\end{APACrefauthors}%
\unskip\
\newblock
\APACrefYear{2010}.
\newblock
\APACrefbtitle {Quantitative analysis of {Italian} texts} {Quantitative
  analysis of {Italian} texts}\ (\BVOL~6).
\newblock
\APACaddressPublisher{L{\"u}denscheid, Germany}{RAM Verlag}.
\PrintBackRefs{\CurrentBib}

\bibitem [\protect \citeauthoryear {%
Visser%
}{%
Visser%
}{%
{\protect \APACyear {2013}}%
}]{%
Visser2013a}
\APACinsertmetastar {%
Visser2013a}%
\begin{APACrefauthors}%
Visser, M.%
\end{APACrefauthors}%
\unskip\
\newblock
\APACrefYearMonthDay{2013}{}{}.
\newblock
{\BBOQ}\APACrefatitle {Zipf's law, power laws and maximum entropy} {Zipf's law,
  power laws and maximum entropy}.{\BBCQ}
\newblock
\APACjournalVolNumPages{New Journal of Physics}{15}{4}{043021}.
\PrintBackRefs{\CurrentBib}

\bibitem [\protect \citeauthoryear {%
Zipf%
}{%
Zipf%
}{%
{\protect \APACyear {1949}}%
}]{%
Zipf1949a}
\APACinsertmetastar {%
Zipf1949a}%
\begin{APACrefauthors}%
Zipf, G\BPBI K.%
\end{APACrefauthors}%
\unskip\
\newblock
\APACrefYear{1949}.
\newblock
\APACrefbtitle {Human behaviour and the principle of least effort} {Human
  behaviour and the principle of least effort}.
\newblock
\APACaddressPublisher{Cambridge (MA), USA}{Addison-Wesley}.
\PrintBackRefs{\CurrentBib}

\end{thebibliography}

\end{document}